\documentclass{article} 
\usepackage{iclr2027_conference,times}

\usepackage{amsthm}
\usepackage{amsmath,amsfonts,bm}
\usepackage{algorithm, algpseudocode}
\usepackage{hyperref}
\usepackage{aliascnt}
\usepackage{cleveref}
\usepackage{graphicx}
\usepackage{booktabs}       
\usepackage{nicefrac}       

\usepackage{tabularray}
\UseTblrLibrary{booktabs}

\def\eqref#1{equation~\ref{#1}}

\def\1{\bm{1}}

\def\eps{{\epsilon}}

\DeclareMathAlphabet{\mathsfit}{\encodingdefault}{\sfdefault}{m}{sl}
\SetMathAlphabet{\mathsfit}{bold}{\encodingdefault}{\sfdefault}{bx}{n}

\def\gO{{\mathcal{O}}}

\newcommand{\E}{\mathbb{E}}

\newcommand{\R}{\mathbb{R}}

\DeclareMathOperator*{\argmax}{arg\,max}

\DeclareMathOperator{\sign}{sign}

\usepackage{amsmath, amssymb, amsfonts, amsthm, mathtools}
\usepackage{soul}
\soulregister\cite7 

\renewcommand{\eps}{\varepsilon}
\newcommand{\parentheses}[1]{\left( #1 \right)}
\newcommand{\lara}[1]{\left\langle #1 \right\rangle}

\newtheorem{theorem}{Theorem}

\newtheorem{assumption}{Assumption}
\crefname{assumption}{assumption}{assumptions}
\Crefname{assumption}{Assumption}{Assumptions}

\crefname{definition}{definition}{definitions}
\Crefname{definition}{Definition}{Definitions}

\newtheorem{lemma}{Lemma}
\crefname{lemma}{lemma}{lemmas}
\Crefname{lemma}{Lemma}{Lemmas}

\newtheorem{remark}[theorem]{Remark}

\definecolor{codeblue}{rgb}{0.21,0.49,0.74}
\definecolor{codegreen}{rgb}{0,0.6,0}

\usepackage{graphicx}
\usepackage[most]{tcolorbox}

\newcommand{\tacoiconjichao}{%
  \raisebox{-0.2\height}{\includegraphics[height=1.2em]{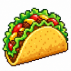}}%
}

\usepackage{mdframed} 
\newcommand{\myNum}[1]{(\emph{#1})}
\newcommand{\smartparagraph}[1]{\vspace{1pt} \noindent {\bf #1}}
\newcommand{\inLineComment}[1]{}

\expandafter\def\expandafter\normalsize\expandafter{%
    \normalsize%
    \setlength\abovedisplayskip{0pt}%
    \setlength\belowdisplayskip{4pt}%
    \setlength\abovedisplayshortskip{-1.2pt}%
    \setlength\belowdisplayshortskip{1pt}%
}
\renewcommand{\eqref}[1]{(\ref{#1})} 

\usepackage{url}
\usepackage{tabularray}

\usepackage{wrapfig}

\usepackage{tabularx}
\usepackage{multirow}
\usepackage{graphicx}
\usepackage{todonotes}

\title{\tacoiconjichao TACO:~{T}ernary {A}bsolute-max {C}olumn-wise {O}ne-sparse Optimizer for LLM Fine-Tuning}

\author{Jichao Jiang$^{1}$, Cristian McGee$^{1}$, El Houcine Bergou$^{2}$,
HanQin Cai$^{1}$ \& Aritra Dutta$^{1}$ \vspace{1mm}\\
$^{1}$University of Central Florida, 
$^{2}$Mohammed VI Polytechnic University
}

\iclrfinalcopy 
\begin{document}

\maketitle

\begin{abstract}
Full-parameter fine-tuning of large language models (LLMs) incurs substantial optimizer state memory overhead, limiting the model sizes that fit on modern GPUs.~Existing approaches either compress optimizer state, abandon first-order gradients, or change the update geometry while retaining dense state. The recently introduced Muon optimizer reduces optimizer memory through matrix-valued updates. Still, its geometry differs from AdamW and can lead to performance degradation when fine-tuning AdamW-pretrained models.~To reduce optimizer memory 
{while preserving competitive downstream accuracy and} computational efficiency in LLM fine-tuning, we propose \textbf{T}ernary \textbf{A}bsolute-max \textbf{C}olumn-wise \textbf{O}ne-sparse optimizer, or \textbf{TACO}, which follows Muon's operator-norm steepest-descent view but takes the geometric route further. TACO computes the exact steepest-descent direction under a dimension-normalized $1\!\to\!1$ operator norm by selecting the sign of the largest magnitude entry in each column of two-dimensional weight matrices.~This retains first-order gradients while making optimizer state memory nearly negligible.~Our practical TACO optimizer maintains only a small set of low precision gradient components per column, reducing persistent optimizer state by $174\times$ relative to AdamW8bit (from $27.7$~GB to $0.16$~GB) and peak training memory by $2.9\times$ (from $80.6$~GB to $27.5$~GB) on OPT-13B, while achieving 
comparable accuracy,
{ $1.75\times$ higher throughput than AdaLOMO, and up to $458\times$ lower collective input payload in multi-GPU distributed training.}
TACO further enables full-parameter fine-tuning of 30-32B-parameter models on a single 80\,GB H100 GPU across multiple model families and tasks.~Our code is publicly available at:
\href{https://github.com/Jichao2357/TACO_optimizer}{\texttt{github.com/Jichao2357/TACO\_optimizer}}.
\end{abstract}


\vspace{-1mm}
\section{Introduction}
\vspace{-1mm}

Full-parameter fine-tuning adapts all model weights and remains a strong baseline for transferring large language models (LLMs) to downstream tasks~\citep{houlsby2019parameter, hu2021lora}; however, it often suffers from memory overhead. For instance, a 13B-parameter model occupies roughly 26\,GB of memory in \texttt{BF16} before allocating any training state. During AdamW \citep{AdamW} fine-tuning, the two \texttt{FP32} moment buffers add another $\sim$\,104\,GB of memory. At the same time, gradients contribute $\sim$\,26\,GB, yielding roughly 156 GB before accounting for activations, temporary workspaces, or an optional \texttt{FP32} master copy; see the \emph{Pareto frontier memory-accuracy trade-off} graph in \Cref{fig:mem_acc-intro}. Thus, at this scale, AdamW's persistent optimizer states alone require roughly $4\times$ the \texttt{BF16} parameter memory. 

Several memory-efficient first-order optimizers reduce this overhead while retaining backpropagation, including state-quantizers (AdamW8bit quantizes optimizer states~\citep{adamw8bit}; FlashAdamW combines 8-bit optimizer states with a compact master-weight representation~\citep{gonzalezblalock2026flashoptim}), factorized second moment methods (Adafactor uses factorized second-moment estimates~\citep{adafactor}), and low-rank gradient methods (GaLore projects gradients and optimizer states into a low-rank subspace~\citep{galore}, and APOLLO-Mini uses a rank-1 subspace to estimate adaptive scaling factors~\citep{zhu2025apollo}),
{as well as backward gradient release methods that fuse parameter updates into backpropagation, such as LOMO~\citep{lv2023full} and AdaLOMO~\citep{lv2023adalomo}.}
In contrast, zeroth-order (ZO) methods avoid backpropagation and approach inference-level memory by estimating updates from forward evaluations~\citep{ STP, IS-STP, SMTP,cai2021zobcd,malladi2023finetuning,MiSTP}. Recent variants incorporate low-rank subspaces or curvature information to improve query efficiency~\citep{yu2025subzero,zhao2025second,kim2025cars}, but still require additional forward queries.

The recently introduced Muon optimizer~\citep{jordan2024muon,bernstein2024oldoptimizer,shah2025muon} asks a different question:~\emph{what steepest-descent direction under a matrix geometry makes the update useful?} Muon treats two-dimensional weights as matrices, retains a single momentum matrix, and orthogonalizes it to approximate a steepest-descent step under \emph{spectral geometry}. This reduces optimizer memory relative to AdamW, although Muon's state remains dense. Additionally, switching an AdamW-pretrained model to Muon during fine-tuning can degrade performance, a phenomenon called \emph{optimizer mismatch} that is attributed to the different update geometries and implicit biases of the two optimizers~\citep{qu2026can}.

Zeroth-Order Muon (ZO-Muon)~\citep{lang2026poweringzerothordertrainingsubspace} improves query efficiency, but as a zeroth-order method, it requires additional forward evaluations. 
Sequential evaluation for ZO methods can keep memory close to inference, but the additional function evaluations can increase computational and communication overhead, especially in distributed training or with larger batches.

Thus, existing approaches, including Muon, expose a fundamental trade-off between the memory footprint of first-order training and the optimization efficiency of standard gradients, prompting a natural question:
\begin{tcolorbox}[
    colback=gray!5,
    colframe=gray!75!black,
    boxrule=0.8pt,
    arc=2mm,
    left=2pt,
    right=2pt,
    top=1pt,
    bottom=1pt
]
\emph{Can full-parameter fine-tuning retain  first-order information and competitive wall-clock efficiency while making persistent optimizer-state memory nearly negligible?}\end{tcolorbox}

\begin{figure}[t]
    \centering
    \includegraphics[width=0.99\linewidth]{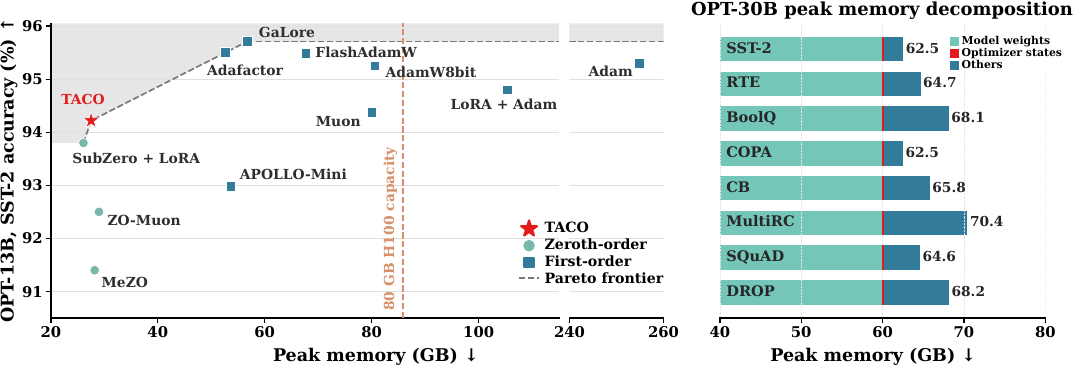}
    \vspace{-1mm}
    \caption{\small{
    \textbf{Performance–memory overview of TACO and other optimizers.} \emph{Left:} Pareto frontier memory-accuracy trade-off for fine-tuning OPT-13B on the SST-2 dataset on an 80 GB H100 GPU. TACO lies at the low-memory end of the Pareto frontier, achieving 94.22\% accuracy with 27.54\,GB of peak training memory. The dashed line connects the Pareto-optimal methods. The vertical line marks the 80\,GB H100 memory limit. Adam achieves 95.3\% accuracy with 254.8\,GB of peak memory
    {(aggregated across $4\times$A100 GPUs in \citet{lang2026poweringzerothordertrainingsubspace}).}
    \emph{Right:} Peak GPU memory decomposition of TACO on OPT-30B across 8 downstream tasks. The bars show model state, optimizer state, and other transient memory across all eight downstream tasks. TACO's optimizer state occupies only 0.267\,GB, less than 0.5\% of peak memory on each task, while the transient memory accounts for the variation across tasks.}
    }\label{fig:mem_acc-intro}
    \vspace{-4mm}
\end{figure}

To answer this, we propose \textbf{TACO}, a \textbf{T}ernary \textbf{A}bsolute-max \textbf{C}olumn-wise \textbf{O}ne-sparse optimizer. 
Using the same operator-norm steepest-descent viewpoint as Muon,
\emph{TACO is not just another memory-efficient optimizer}; TACO's conceptual centerpiece is an optimization geometry whose exact steepest-descent direction has extreme structural sparsity. Under a dimension-normalized $1\!\to\!1$ operator norm, the corresponding dual geometry makes the steepest-descent direction separable across columns, and each column selects the sign of its largest-magnitude gradient coordinate. Thus, an $m\times n$ update contains at most $n$ nonzeros.


Although the TACO steepest-descent direction is one-sparse per column, minibatch gradients can make the selected coordinates unstable across training. A conventional dense gradient history can stabilize coordinate selection but erodes the memory advantage. TACO instead retains only a small, low-precision set of dynamic gradient heavy hitters per column. This sparse gradient history keeps persistent optimizer-state memory negligible even as model size grows.
For example, TACO can fine-tune OPT-30B on a single 80\,GB H100 GPU, with peak GPU memory of 62.5\,GB on SST-2, 64.7\,GB on RTE, and 68.1\,GB on BoolQ; see Peak GPU memory decomposition in Figure~\ref{fig:mem_acc-intro}. Adam does not fit this model on a single H100 80\,GB GPU. Across our study, TACO enables full-parameter fine-tuning of models up to 32B parameters on a single 80\,GB H100 and transfers across multiple model families and tasks.

Taken together, our main contributions are:

\myNum{i} \smartparagraph{Geometry-induced sparse optimization.} We propose TACO, a new first-order optimizer for memory-efficient full-parameter LLM fine-tuning. TACO is derived from the exact steepest-descent solution under a dimension-normalized $1\!\to\!1$ operator norm, whose geometry naturally yields a ternary, column-wise Top-1 update. Thus, TACO's sparsity is intrinsic to its optimization geometry rather than imposed by post-hoc gradient sparsification; see \Cref{sec:TACO}.

\myNum{ii} \smartparagraph{Convergence and fine-tuning theory.} We prove that vanilla TACO reaches an $\varepsilon$-stationary point in $\mathcal{O}(\varepsilon^{-2})$ iterations under standard smoothness assumptions
{and extend this guarantee to minibatch stochastic gradients. 
}We further show that, in our theoretical setting, TACO shares the continuation geometry and limiting solution of an Adam proxy while differing from Muon, and satisfies standard $\mu$P spectral scaling under balanced winner counts. These results connect TACO to optimizer mismatch and the learning rate transfer problem; see \Cref{sec:TACO,sec:geometry}.

\myNum{iii} \smartparagraph{TACO-specific memory-efficient design.} TACO does not require an accurate dense model of the gradient history, but only enough information to recover the dominant coordinate in each column. Exploiting this structure, we develop a practical implementation based on dynamic {\texttt{FP8} or \texttt{FP16}} sparse gradient histories and immediate gradient processing, reducing persistent state from $\mathcal{O}(mn)$ to $\mathcal{O}(n)$ while retaining first-order gradients, the column-wise Top-1 update, 
{and $\mathcal{O}(n)$ per-matrix parameter updates for communication efficient distributed training};
see \Cref{sec:practical implementation}.

\myNum{iv} \smartparagraph{Empirical validation at scale.}
We evaluate TACO across model scales, architectures, and downstream tasks, 
{including controlled backward gradient release comparisons (Adafactor, AdamW8bit with autograd hook, and AdaLOMO), matched wallclock optimization budgets, precision diagnostics, and multi-GPU distributed communication benchmarks}.
TACO achieves competitive task performance with substantially reduced training memory, enables full-parameter fine-tuning of OPT-30B and Qwen3-32B on a single 80\,GB H100 GPU, and transfers across OPT, Qwen3, Pythia, Llama, and Mistral model families. Ablations further characterize the roles of the gradient history EMA decay factor, heavy hitter threshold, update sparsity, and optimizer step budget; {see \Cref{sec:experiment} and \Cref{,sec:extended_experiments}.}

\vspace{-2mm}
\section{Related Work} 
\vspace{-2mm}

\smartparagraph{Zeroth-order optimization} has become a practical way to reduce the memory cost of LLM fine-tuning. These methods construct an update direction from the change in loss under a parameter perturbation. MeZO applies this idea by regenerating its perturbations in place using a seeding trick~\citep{malladi2023finetuning}. Sparse MeZO reduces the number of perturbed parameters, preserving sparsity~\citep{liu2025sparse}. ZO-Muon projects the estimated gradient into a lower-dimensional subspace and applies polar orthogonalization to recover more useful structure from the noisy estimate~\citep{lang2026poweringzerothordertrainingsubspace}. HiZOO instead incorporates diagonal Hessian information into zeroth-order
optimization, using curvature-aware preconditioning to improve convergence~\citep{zhao2025second}.


\smartparagraph{Memory-efficient first-order optimization} methods retain backpropagation while reducing parameter-associated training state. Adafactor represents second moments using row and column factors~\citep{adafactor}, while CAME adds a confidence-guided correction to improve the stability of this representation~\citep{came}. Alternatively, 8-bit Adam uses block-wise quantization to store optimizer statistics at lower precision~\citep{adamw8bit}. FlashOptim combines quantized optimizer states with a compact master-weight representation and provides FlashAdamW as a drop-in AdamW variant~\citep{gonzalezblalock2026flashoptim}. GaLore projects gradients and optimizer states into a lower-dimensional space~\citep{galore}, whereas APOLLO approximates adaptive gradient scaling through a random low-rank auxiliary space; its APOLLO-Mini variant reduces this space to rank one and applies tensor-wise scaling~\citep{zhu2025apollo}. {Complementarily, LOMO~\citep{lv2023full} and AdaLOMO~\citep{lv2023adalomo} fuse parameter updates into the backward pass to release gradients immediately without storing full-model gradient buffers.}

\smartparagraph{Operator-norm and linear-minimization-oracle (LMO) optimization.}
\cite{bernstein2024oldoptimizer} formulate matrix optimization as steepest-descent under
operator norms and characterize the $1\!\to\!p$ family through maximum column
norms. \cite{pethick2025normlmo} develop a stochastic
LMO framework and instantiate it with layer-aware
operator geometries in Scion. 

\smartparagraph{Sign-based methods.} In gradient compression, TernGrad maps gradient values to \(\{-1,0,+1\}\) to reduce the communication bottleneck in distributed training~\citep{wen2017terngrad}, while signSGD retains only the coordinate-wise gradient signs~\citep{signsgd}; see GRACE~\citep{xu2021grace}, a compressed communication framework that places these methods under one umbrella. SignMuon~\citep{signmuon} communicates the signs of locally constructed Muon directions. 



\vspace{-1mm}
\section{\protect\tacoiconjichao TACO:~\textbf{T}ernary \textbf{A}bsolute-max \textbf{C}olumn-wise \textbf{O}ne-sparse Optimizer}
\label{sec:TACO}
\vspace{-1mm}


\smartparagraph{Setup and notation.}  
We consider the minimization problem 
\begin{equation}
\label{eqn:minimization_problem}
    \min_{W\in\R^{m\times n}}F(W).
\end{equation}
Let $G=\nabla F(W)$. We use $d_j$ and $g_j$ to denote the $j^{\rm th}$ column of matrices $D$ and $G$, respectively, and define $\lara{G,D}_F=\operatorname{tr}(G^\top D)$. For a vector $x\in\R^d$ and $r\in [1,\infty]$, we define the dimension-normalized $r$-norm as $\|x\|_{r,d}=d^{-1/r}\|x\|_r$, with the convention $d^{-1/\infty} = 1$. For a linear map $D:\R^n\to\R^m$, we define the induced operator norm,
\(
\|D\|_{(p,n\to q,m)}=\sup_{x\ne0}\|Dx\|_{q,m}/\|x\|_{p,n}.
\)
Here, the domain $\R^n$ is equipped with $\|\cdot\|_{p, n}$ and the codomain $\R^m$ with $\|\cdot\|_{q,m}$.
The indices in $(p,n\to q,m)$ record the norm orders and dimensions of the input and output spaces; here $n$ and $m$ determine normalization factors, not additional tensor indices. We use $\odot$ to denote the element-wise Hadamard product and, with a slight abuse of notation, let $\operatorname{sign}(\cdot)$ act element-wise on vectors and matrices. For $q\in[1,\infty]$, we write $\|A\|_{1\to q}=\max_{\|x\|_1\leq 1}\|Ax\|_q$.

\smartparagraph{From Muon to TACO.} A first-order Taylor approximation of \eqref{eqn:minimization_problem} gives
\begin{equation*}
    F(W + \eta D) = F(W) + \eta\lara{G, D}_F + \gO(\eta^2).
\end{equation*}
From this steepest-descent viewpoint, Muon solves the local minimization problem
\begin{equation}
\label{Muonlocal_minimization_problem}
    \min_D \lara{G,D}_F,\quad \|D\|_{(2, n\to 2, m)}\le 1,
\end{equation}
where the geometry used to constrain $D$ determines the steepest local descent direction. 
In this view, Muon uses the normalized $2\!\to\!2$ (spectral) operator norm; if $G$ has the SVD decomposition  $G=U\Sigma V^\top$, then  $D$ is proportional to $UV^\top$. 

TACO makes a different geometric choice. Since the $\ell_1$-norm encourages sparsity, we equip both spaces with normalized $\ell_1$-norms, for which
\begin{equation*}
    \|D\|_{(1,n\to 1,m)} = \frac{n}{m} \|D\|_{1\to 1} = \frac{n}{m} \max_j \|d_j\|_1,
\end{equation*}
and we reformulate \eqref{Muonlocal_minimization_problem} as
\begin{tcolorbox}[
   colback=gray!5,
    colframe=gray!75!black,
    boxrule=0.8pt,
    arc=2mm,
    left=2pt,
    right=2pt,
    top=1pt,
    bottom=1pt
]
\begin{equation}
\label{TACO:eqn:local_minimization_problem}
\min_D \lara{G, D}_F,\quad \|D\|_{(1,n\to 1,m)}\le 1.
\end{equation}
\end{tcolorbox}
The maximum-column identity is the $p=1$ specialization of the general
$1\!\to\!p$ characterization; TACO's
specific choice is to use this geometry as the basis of a column-wise
one-sparse, low-state optimizer.
Thus, $1\!\to\!1$ is not a global element-wise $\ell_1$ constraint. It gives every column its own $\ell_1$ budget. Its dual norm under the Frobenius pairing is
\begin{equation*}
    \|G\|_{(1,n \to 1,m)^*} = \frac{m}{n}\sum_j \|g_j\|_\infty.
\end{equation*}
Because the constraint separates across columns, each column spends its budget on an extreme point of an $\ell_1$ ball, its largest-magnitude gradient coordinate. Let $i_j\in \argmax_i |G_{ij}|, j\in [n]$, with any fixed rule for breaking ties, define the operator $\mathcal{S}: \R^{m\times n} \to \R^{m\times n}$ as
\begin{equation}
\label{TACO:eqn:TACO_operator}
    \mathcal{S}(G)_{ij}:= \begin{cases}
        \sign(G_{ij}), &i = i_j,\\
        0, & \text{otherwise}.
    \end{cases}
\end{equation}
Therefore, an exact minimizer of \eqref{TACO:eqn:local_minimization_problem} is
\begin{equation*}
    D^\star = -\frac{m}{n} \mathcal{S}(G).
\end{equation*}
TACO's geometry produces a \textbf{ternary, column-wise, and one-sparse} update; an $m\times n$ update has at most $n$ nonzeros.~
The factor \(m/n\) is not an additional layerwise scaling hyperparameter; it follows directly from dimension normalization.
This distinguishes TACO from Top-1 sparsification in conventional optimization \citep{sahu2021rethinking, dutta2020discrepancy}. TACO's sparsity follows from the \emph{exact closed-form solution of the local problem.}
\begin{tcolorbox}[
    colback=gray!5,
    colframe=gray!75!black,
    boxrule=0.8pt,
    arc=2mm,
    left=2pt,
    right=2pt,
    top=1pt,
    bottom=1pt
]
Given $W_t$ and $G_t=\nabla F(W_t)$, vanilla TACO computes $D_t=\frac{m}{n}\mathcal{S}(G_t)$ and updates the sequence $\{W_t\}_{t=0}^T$ as
\begin{equation}\label{eq:vanilla TACO update}
    W_{t+1}=W_t-\eta_t D_t.
\end{equation}
\end{tcolorbox}
The term vanilla distinguishes the history-free TACO iteration from the noise-robust practical implementation in \Cref{sec:practical implementation}.

\smartparagraph{Convergence of Vanilla TACO.}
We state the convergence guarantee of vanilla TACO using the same primal-dual norm pair that defines the update.
\begin{assumption}{(Lower boundedness and smoothness)}\label{TACO:L-smooth_assumption}
The objective function is lower bounded and $L$-smooth. That is, $F_{\star}:=\inf_W F(W)>-\infty$, and for all $W_1,W_2\in\R^{m\times n}$, the following holds: 
    $\| \nabla F(W_1) - \nabla F(W_2)\|_{(1,n\to 1,m)^*} \le L \|W_1 - W_2\|_{(1,n \to 1,m)}.$
\end{assumption}

\begin{theorem}{(Nonconvex convergence of vanilla TACO)}\label{thm:taco-convergence}
Let \Cref{TACO:L-smooth_assumption} hold and define $\Delta_F:= F(W_0)-F_{\star}$. The iterates $\{W_t\}_{t=0}^{T}$ in \eqref{eq:vanilla TACO update} with step sizes $\eta_t>0$ satisfy
\begin{equation*}
    \min_{0\le t<T}\|\nabla F(W_t)\|_{(1,n\to1,m)^*}
    \le
    \frac{\Delta_F+\frac{L}{2}\sum_{t=0}^{T-1}\eta_t^2}
         {\sum_{t=0}^{T-1}\eta_t}.
\end{equation*}
In particular, with $\eta=\sqrt{\tfrac{2\Delta_F}{LT}}$, we have
    $
        \min_{0\le t<T} \|\nabla F(W_t)\|_{(1,n \to 1,m)^*}\le \sqrt{\frac{2L\Delta_F}{T}}.
    $
\end{theorem}
We call $W$ an $\eps$-stationary point in the TACO geometry if $\|\nabla F(W)\|_{(1,n \to 1, m)^*} \le \eps$. \Cref{thm:taco-convergence} shows that vanilla TACO reaches an $\eps$-stationary point in $\mathcal{O}(\eps^{-2})$ iterations.



{
\begin{assumption}\label{rebuttal:assumption:unbias_estimate}
    {(Unbiased stochastic gradients)}
    At each iteration $t$, the stochastic gradient is defined as $\tilde G_t = \frac{1}{B}\sum_{b=1}^{B} \nabla f(W_t; \xi_{t,b})$, where ${\xi_{t,1},\ldots,\xi_{t,B}}$ are
independent and have the same distribution as $\xi$. 
\end{assumption}

\begin{assumption}\label{rebuttal:assumption:bounded_variance}
    {(Bounded variance)}
    There exists $\nu > 0$ such that for every $W$, $\E_\xi \|\nabla f(W; \xi) - \nabla F(W)\|_F^2 \le \nu^2$.
\end{assumption}
\begin{lemma}[Mini-batch stochastic gradient properties]
\label{rebuttal:lemma:minibatch_gradient}
Under Assumptions~\ref{rebuttal:assumption:unbias_estimate} 
and \ref{rebuttal:assumption:bounded_variance},
the mini-batch estimator satisfies
\begin{equation*}
\E_t[\tilde G_t] = G_t, \qquad
\E_t\|\tilde G_t-G_t\|_F^2 \le \frac{\nu^2}{B}.
\end{equation*}
\end{lemma}

\begin{theorem}\label{rebuttal:thm:stochastic gradients vannila TACO}
    {(Nonconvex convergence of vanilla TACO with stochastic gradients)}
    Let \Cref{TACO:L-smooth_assumption,rebuttal:assumption:unbias_estimate,rebuttal:assumption:bounded_variance} hold. For any initial $W_0$, fix $T$ iterations with constant learning rate $\eta > 0$ and minibatch size $B$, and let $\{W_t\}$ be generated by \eqref{eq:vanilla TACO update} with stochastic gradients $D_t = \frac{m}{n}\mathcal{S}(\tilde G_t)$. Denote $\Delta_F = F(W_0) - F_\star$. Then
    \begin{equation}
        \min_{0\le t< T} \E \|\nabla F(W_t)\|_{(1,n\to 1,m)^*} \le \frac{\Delta_F}{\eta T} + \frac{2m\nu}{\sqrt{nB}} + \frac{\eta L}{2}.
    \end{equation}
    In particular, choosing $\eta = \sqrt{\frac{2\Delta_F}{LT}}$, setting $T\ge \frac{8L\Delta_F}{\eps^2}$ and $B\ge \frac{16m^2\nu^2}{n\eps^2}$ guarantees 
    \begin{equation}
        \min_{0\le t <T} \E\|\nabla F(W_t)\|_{(1,n\to 1,m)^*}\le \eps.
    \end{equation}
\end{theorem}
}

\section{Continuation Geometry, Implicit Bias, and $\mu$P-scaling of TACO}\label{sec:geometry}





Recent work shows that Adam and Muon can exhibit different fine-tuning behavior because their update geometries induce different implicit biases \citep{qu2026can}. We study where TACO lies relative to these geometries in a simplified fixed-input setting, and then analyze how the spectral norm of its update scales with matrix dimensions.



Let $z\in\mathbb{R}^n, z\neq 0$, denote a fixed input. Ties in TACO's column-wise selections are resolved using a fixed deterministic rule. For $W_0 \in \mathbb{R}^{m\times n}$ and a vector $w \in \mathbb{R}^n$, let the continuation set generated by $w$ at $W_0$ be represented by $\mathcal{C}(W_0,w):=\{W_0+aw^\top : a \in \mathbb{R}^m\}$.

We include the Adam and Muon proxies in \eqref{eq:adam-proxy} and~\eqref{eq:muon-proxy}, respectively; see \Cref{sec:theory-extension}. To make the optimizer comparison precise, we consider a setting in which the objective depends on a weight matrix only through its action on a fixed input and ask how different update geometries move through parameter space while producing the same target output. 

\begin{assumption}\label{ass:coherent}{(Coherent objective)} The objective depends on $W$ through the layer output, i.e., $F(W)=\phi(Wz)$ for a differentiable function $\phi:\mathbb{R}^m\to\mathbb{R}$.
\end{assumption}


\begin{assumption}\label{ass:attained}
{(Common target output)}
Starting from the same $W_0$, all compared methods converge to limits $W^\star$ satisfying
$
W^\star z=u^\star
$
for the same optimal output $u^\star$.
\end{assumption}

With these conditions in place, we now ask how each optimizer moves through parameter space. We first show that, despite very different update rules, TACO and Adam follow the same continuation geometry and converge to the same parameter-space solution. We then contrast this behavior with Muon, whose spectral geometry induces a different continuation path.

\begin{theorem}({Adam and TACO continuation equivalence})\label{thm:adam-taco}
Let \Cref{ass:coherent,ass:attained} hold, and let the Adam Proxy follow \eqref{eq:adam-proxy}. Then both the TACO and Adam sequences lie in $\mathcal{C}(W_0,\operatorname{sign}(z))$ for all $t\geq 0$, and
$W_{\rm TACO}^\star=W_{\rm Adam}^\star=W_0+\frac{(u^\star-W_0z)\operatorname{sign}(z)^\top}{\|z\|_1}.$
\end{theorem}

\begin{theorem}({Muon continuation separation})\label{thm:muon-sep}
Let \Cref{ass:coherent,ass:attained} hold and let the Muon proxy follow \eqref{eq:muon-proxy}. Further let $r=u^\star-W_0z$. Then Muon’s iterates lie in $\mathcal{C}(W_0,z)$ and
$W_{\rm Muon}^\star=W_0+\frac{rz^\top}{\|z\|_2^2}, W_{\rm Muon}^\star-W_{\rm TACO}^\star=r\left(\frac{z}{\|z\|_2^2}-\frac{\operatorname{sign}(z)}{\|z\|_1}\right)^\top.$
\end{theorem}

\Cref{thm:adam-taco,thm:muon-sep} demonstrate that, under this setting, TACO behaves more like Adam than Muon when fine-tuning a model, even though TACO was motivated by the same geometric framework as Muon. In our setting, TACO and Adam follow the same family of parameter updates and arrive at the same solution. In contrast, Muon generally moves in a different direction and ends at a different solution. This matters because models are usually pretrained with AdamW, so preserving Adam-like behavior during fine-tuning
{provides a conditional geometric mechanism that}
may help avoid the optimizer mismatch problem that appears when switching to Muon.

\smartparagraph{Theoretical $\mu$P.} The $\mu$P problem seeks update scalings that preserve well-behaved optimization dynamics as model width increases, allowing the learning rate to transfer across model scales rather than requiring repeated hyperparameter tuning; Muon is one solution to the framework~\citep{bernstein2025deriving}. Since TACO is derived from the same steepest-descent perspective under a different operator geometry, we ask whether its update preserves the spectral scaling required for $\mu$P transfer.

\begin{assumption}\label{assum:winner_balance}{(Balanced TACO Winners)}
For $G\in\mathbb{R}^{m\times n}$, let $c_i$ denote the number of columns whose TACO winner lies in row $i$. We assume $\max_i c_i=\Theta(n/m)$ for all iterates $W_t$.
\end{assumption}

\begin{theorem}{(TACO $\mu$P-scaling)}\label{thm:taco_mup}
Let \Cref{assum:winner_balance} hold. Then the TACO direction $D=\frac{m}{n}\mathcal{S}(G)$ satisfies $\|D\|_2=\frac{m}{n}\sqrt{\max_i c_i},$ for every $G$. Hence $D$ satisfies the standard $\mu$P spectral scaling
$
\|D\|_2=\Theta\left(\sqrt{\frac{m}{n}}\right).
$

\end{theorem}

\Cref{thm:taco_mup} shows that TACO satisfies the standard $\mu$P spectral scaling law when the largest-magnitude coordinates selected across columns remain reasonably balanced across rows, rather than becoming disproportionately concentrated in only a few rows. Hence, the theorem provides a theoretical basis for learning rate transfer across model widths. We complement this result with empirical $\mu$P experiments in \Cref{sec:more_experiment_results}, where the same learning rate remains optimal across model scales. 

Together with the continuation analysis above, these theoretical developments position TACO as an unusually favorable combination: 
it inherits the reduced hyperparameter sensitivity associated with $\mu$P scaling,
follows the same continuation geometry as AdamW under our theoretical setting, and retains these properties while requiring almost no persistent optimizer state{, as well as empirical row-winner concentration and normalized spectral ratio diagnostics across model widths in \Cref{subsec:winner-diagnostics}}.

\vspace{-1mm}
\section{Practical Design of TACO}\label{sec:practical implementation}
\vspace{-1mm}

Although vanilla TACO updates are sparse by design, applying $\mathcal{S}$ directly to a single minibatch gradient makes the selected coordinates sensitive to minibatch noise. A conventional dense gradient history can stabilize coordinate selection but also reintroduce the dense $mn$-entry state TACO was meant to avoid. Exploiting TACO's special geometric update structure, our implementation resolves this tension with a sparse gradient history, low-precision state, and immediate gradient processing.
This motivates the following design principle:
\vspace{-1mm}
\begin{tcolorbox}[
    colback=gray!5,
    colframe=gray!75!black,
    boxrule=0.8pt,
    arc=2mm,
    left=2pt,
    right=2pt,
    top=1pt,
    bottom=1pt
]
TACO does not need an accurate model of the gradient history. TACO only needs enough information to recover its dominant coordinate in each column.
\end{tcolorbox}


\smartparagraph{Gradient history and dynamic heavy hitters.}
A dense exponential moving average (EMA) is
    $\tilde M_t = \mu \tilde M_{t-1} + (1-\mu) G_t, \mu = 0.95.$
However, the next TACO direction only needs a reliable winner: the entry with the largest absolute value in each column, rather than a dense reconstruction of $\tilde M_t$. Let $\mathcal{H}_k(X)$ retain the $k$ largest-magnitude entries of every column of $X$ and set all remaining entries to zero. Practical TACO maintains the following update rule:
\begin{equation*}
    M_t = \mathcal{H}_k\!\left(\mu M_{t-1} + (1-\mu)G_t\right),
    \qquad \mu=0.95,\quad k=16,
\end{equation*}
where the unretained entries of the sparse $M_{t-1}$ are treated as zero. The update remains $D_t=\frac{m}{n}\mathcal{S}(M_t)$. In our experiments, we set $k=16$ to control how many entries are retained in each column for future winner selection; it does \emph{not} change the column-wise Top-1 TACO update.
For additional discussion of sparse gradient history approximation,
see {Theorem~\ref{thm:taco-exact-history} in \Cref{sec:theory-extension}} and ablations in \Cref{sec:more_experiment_results}.
The retained history values are stored in \texttt{FP8 E4M3} format, which uses four exponent bits and three mantissa bits~\citep{micikevicius2022fp8} {(or \texttt{FP16} in TACO-FP16)}, together with \texttt{int32} for row indices. Winner selection depends primarily on relative magnitude and sign and therefore requires less numerical fidelity than an adaptive moment used to rescale every coordinate, {although unscaled \texttt{E4M3} can underflow small gradients below $2^{-10}$; storing retained entries in \texttt{FP16} eliminates history underflow at negligible extra memory (+29.7\,MB on OPT-13B); see \Cref{subsec:controlled-hook-exp,subsec:precision-diagnostics} for TACO-FP16 evaluations and layerwise update diagnostics}.

\smartparagraph{Persistent state memory.}
Each retained entry uses one byte for its \texttt{E4M3} value and four bytes for its \texttt{int32} row index. Thus, for an $m\times n$ matrix, TACO stores $5kn$ bytes, compared with $mn b_{\mathrm{dense}}$ bytes for one dense state matrix whose entries use $b_{\mathrm{dense}}$ bytes. When $m\leq 2^{15}$, \texttt{int16} indices suffice, reducing the history storage to $3kn$ bytes. But some model matrices have more rows than signed \texttt{int16} can index, so we use \texttt{int32}. With \texttt{int32}, the relative history cost is $5k/(m b_{\mathrm{dense}})$, which decreases as the matrix height $m$ grows and is small when $k\ll m$. Thus, TACO reduces persistent state memory complexity from $\gO(mn)$ to $\gO(kn)$.

\definecolor{TACOpurple}{HTML}{8A6596}
\definecolor{TACOgreen}{HTML}{78B7AA}
\definecolor{TACOblue}{HTML}{327B9A}
\definecolor{TACOorange}{HTML}{D89068}

\begin{algorithm}[t]
\small
    \caption{Practical implementation of TACO with
    \colorbox{TACOorange}{\textcolor{white}{gradient history}}
    and 
    \colorbox{TACOgreen}{\textcolor{white}{dynamic heavy hitters}}
    }
    \label{TACO:alg:TACO_optimizer}
    \begin{algorithmic}[1]
    \Require $W_0\in\R^{m\times n}$; learning rate schedule $(\eta_t)_{t\ge0}$; gradient history EMA decay factor $\mu=0.95$; heavy hitter threshold $k=16$; optimizer step budget $T$.
    \Ensure Final parameters $W_T$.
    \State Initialize an empty sparse gradient history $M_{-1}$
    \For{$t = 0,\ldots,T-1$}
        \State Compute $G_t\gets\nabla F(W_t)$ \Comment{\texttt{Backpropagation}}
        \State $G_t \gets\colorbox{TACOorange}{$\displaystyle\textcolor{white}{\mu M_{t-1}+(1-\mu)G_t}$}$ \Comment{\texttt{In-place history update}}
        \State $M_t \gets\colorbox{TACOgreen}{$\displaystyle\textcolor{white}{\mathcal{H}_k}$}\left(G_t\right)$ \Comment{\texttt{Dynamic heavy hitters}}
        \State Store the retained values of $M_t$ in \texttt{E4M3} {or \texttt{FP16}} and their row indices in \texttt{int32}
        \State $W_{t+1}\gets W_t-\eta_t\frac{m}{n}
        \mathcal{S}(M_t)${, optionally with stochastic rounding.}
        \Comment{$\mathcal{S}$ \texttt{is defined in \eqref{TACO:eqn:TACO_operator}}}
        \State Release the dense $G_t$
    \EndFor
    \State \Return $W_T$
    \end{algorithmic}
\end{algorithm}
\begin{figure}[t]
    \centering
    \;
    \begin{minipage}[c]{0.24\linewidth}
        \centering
        \begin{tikzpicture}[
            scale=0.8, transform shape,
            node distance=0.3cm, 
            box/.style={
                draw,
                rounded corners=1pt,
                align=center,
                inner xsep=5pt,
                inner ysep=4pt,
                font=\small,
                text width=4cm 
            },
            arrow/.style={->, >=stealth, thin, shorten >=0.05cm, shorten <=0.05cm}
        ]
    
        \node[box] (grad) {
            Autograd\\
            $p.\texttt{grad} = G_t$
        };
    
        \node[box, below=of grad] (ema) {
            In-place EMA\\
            $G_t \leftarrow \mu M_{t-1}+(1-\mu)G_t$
        };
    
        \node[box, below=of ema] (threshold) {
            Hard threshold\\
            $M_t \leftarrow \mathcal{H}_{16}(G_t)$
        };
    
        \node[box, below=of threshold] (update) {
            Release $G_t$ and update\\
            $W_{t+1} =W_t-\eta_t\frac{m}{n}\mathcal{S}(M_t)$
        };
    
        \draw[arrow] (grad) -- (ema);
        \draw[arrow] (ema) -- (threshold);
        \draw[arrow] (threshold) -- (update);
    
        \end{tikzpicture}
        \label{fig:taco-workflow}
    \end{minipage}
    \;
    \vspace{-2mm}
    \hfill 
    \begin{minipage}[c]{0.725\linewidth}
        \centering
        \includegraphics[width=0.99\linewidth]{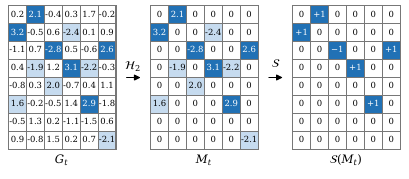}
        \label{fig:hk-selector}
    \end{minipage}
    \vspace{-3mm}
    \caption{\small{\textbf{Illustrative workflow of the memory-efficient practical implementation of TACO.} We use $\mathcal{H}_2$ for illustration; 
    in our experiments, we use $\mathcal{H}_{16}$.~In this example, we highlight the largest- and second-largest-magnitude elements for $\mathcal{H}_2$ in each column. Here, $G_t$ is the dense gradient via \texttt{autograd}, $M_t$ is obtained via $\mathcal{H}_{2}(G_t)$ by retaining the largest- and second-largest-magnitude elements in each column, and finally, the operator $\mathcal{S}$ is used, which keeps the sign of the largest-magnitude element in each column. Only the sparse gradient history $M_t$ persists across iterations.}}
    \vspace{-3mm}
\end{figure}

\smartparagraph{Backward pass memory management.}
TACO is applied to matrix-valued parameters, while scalar and vector parameters use an auxiliary AdamW optimizer. PyTorch \texttt{autograd hooks}~\citep{Pytorch} make each matrix gradient $G_t$ available as soon as it is computed. TACO updates $G_t$ in-place using the sparse gradient history, extracts $\mathcal{H}_k(G_t)$ as the new history, applies the parameter update, and releases $G_t$. We set gradient accumulation to~1, so dense gradients do not persist across microbatches. Gradient checkpointing reduces activation memory. Together, these techniques bring peak training memory close to inference memory. Peak training memory includes model parameters, sparse optimizer state, checkpointed activations, and the current gradient. \texttt{Autograd hooks} and checkpointing are standard systems techniques rather than contributions of this work.
{When equipped with identical post accumulate \texttt{autograd hooks}, low state optimizers (such as Adafactor + autograd hooks and AdaLOMO) reach the same activation dominated VRAM floor ($\sim 27$-$31$\,GB on OPT-13B). However, TACO's columnwise Top-1 direction modifies at most $n$ entries per matrix step ($\gO(n)$ updates vs.\ $\gO(mn)$ for dense updates), directly enabling low volume distributed communication; see \Cref{subsec:controlled-hook-exp,subsec:distributed-exp}.}

\vspace{-2mm}
\section{Experiments} \label{sec:benchmarking} \label{sec:experiment}
\vspace{-2mm}

\smartparagraph{Implementation details and models used.} We evaluate TACO for full-parameter fine-tuning across five model families: OPT, Qwen3, Pythia, Llama, and Mistral, covering model sizes from 1.3B to 32B parameters. Our experiments include SST-2, RTE, BoolQ, MultiRC, SQuAD, DROP, COPA, and CB, spanning classification, multiple-choice reasoning, and question answering. We first compare TACO with a range of optimizers on OPT-13B, then study its scaling behavior across the OPT and Qwen3 model families and its transfer across different model architectures. Finally, we summarize several ablations characterizing TACO's optimizer step budget, gradient history EMA decay factor, heavy hitter threshold, and update sparsity. Unless otherwise specified, TACO uses a shared hyperparameter configuration across the non-ablation experiments; see complete experimental configurations, dataset construction, prompt formats, hyperparameters, and results in \Cref{sec:experiment_setup}. 
{
We report extended evaluations covering controlled gradient release baselines, TACO-FP16, matched wallclock budgets, precision and winner concentration diagnostics, and 4-GPU distributed training in \Cref{sec:extended_experiments}.
}

\smartparagraph{Memory, trained model quality, and runtime on OPT-13B.}
\Cref{fig:opt13b task by task comparison} compares TACO with memory-efficient and conventional first-order optimizers across eight downstream tasks. TACO consistently requires substantially less peak GPU memory while maintaining competitive task performance. The SST-2 efficiency comparison in \Cref{fig:opt13b efficiency} isolates this advantage: beyond the $25.7$\,GB required for model weights, TACO uses only $1.8$\,GB of additional training memory, including just~$0.16$\,GB of persistent optimizer state. In comparison, Adafactor, GaLore, and FlashAdamW require $26.9$, $31.0$, and $42.0$\,GB of additional memory, respectively, resulting in peak memory of $52.6$, $56.7$, and $67.7$\,GB versus $27.5$\,GB for TACO. 
Although Adafactor has the smallest persistent optimizer state, its substantially larger transient memory results in much higher peak GPU memory than TACO.
TACO reaches 94.2\% SST-2 accuracy at 276 tokens/second, compared with 95.5\% at 350~tokens/second for Adafactor, 95.7\% at 31~tokens/second for GaLore, and 95.5\% at 547~tokens/second for FlashAdamW. 
TACO therefore occupies the low-memory end of the accuracy–memory trade-off while retaining competitive performance and practical throughput; 
{
see \Cref{subsec:controlled-hook-exp} for controlled comparisons against TACO-FP16, Adafactor + autograd hooks, AdamW8bit + autograd hooks, and AdaLOMO under unified hyperparameters and matched wallclock budgets.}

\begin{figure}[t!]
    \centering
    \includegraphics[width=0.99\linewidth]{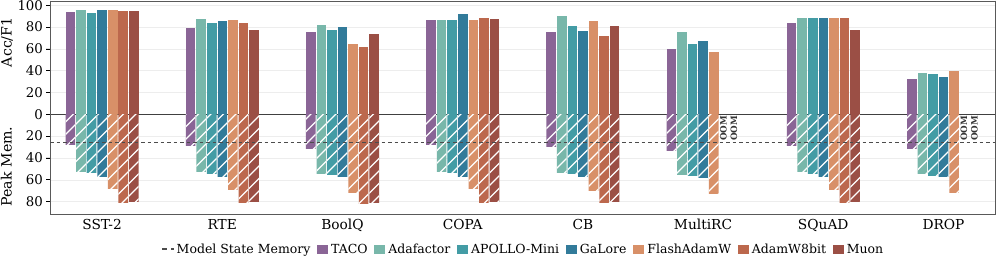}
        \vspace{-1mm}
    \caption{\small{\textbf{
    OPT-13B performance and peak-memory comparison across downstream tasks.} Upper bars report test accuracy or F1; lower mirrored bars report peak GPU memory. \textbf{OOM} denotes configurations that do not fit within the available GPU memory. Across all tasks, TACO consistently requires the least GPU memory while maintaining competitive performance; see \Cref{tab:e1-performance} in \Cref{sec:more_experiment_results} for full results.}
    }
    \label{fig:opt13b task by task comparison}
    \vspace{-1mm}
\end{figure} 
\begin{figure}[t!]
    \centering
    \includegraphics[width=0.99\linewidth]{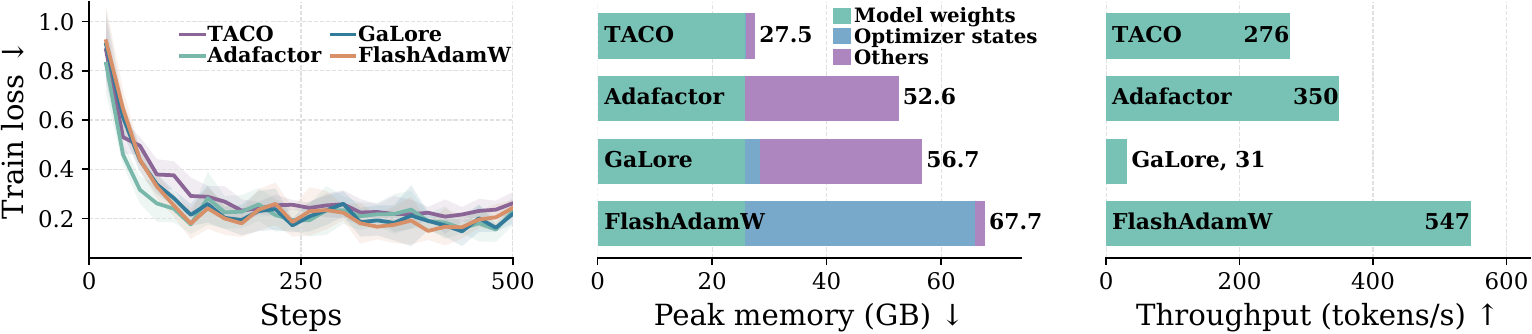}
        \vspace{-1mm}
    \caption{\small \textbf{Efficiency comparison} in terms of optimization trajectory, peak memory decomposition, and training throughput for fine-tuning OPT-13B on SST-2. For the Train loss vs. Steps plot, we show only the first 500 steps, as all optimizers plateaued; the shaded region shows \(\pm\) sample standard deviation across five training seeds; see \Cref{tab:opt13b_sst2_efficiency} in \Cref{sec:more_experiment_results} for full results.
    }\label{fig:opt13b efficiency}
    \vspace{-4mm}
\end{figure}


\smartparagraph{Scaling across model families.} Figure~\ref{fig:model_scaling} shows that TACO remains trainable within the memory limit of a single 80 GB H100 as OPT scales from 1.3B to 30B parameters and Qwen3 scales from 8B to 32B parameters. Within OPT, SST-2 accuracy remains broadly stable, RTE generally improves with scale, and BoolQ is non-monotonic. Qwen3 similarly maintains strong performance across model sizes. Peak memory scales with model size and stays within the capacity of a single H100 80\,GB GPU, even for OPT-30B and Qwen3-32B. 
Importantly, these experiments use the same TACO hyperparameter configuration across model sizes rather than model-specific tuning.

\begin{figure}[t]
    \centering
    \includegraphics[width=0.99\linewidth]{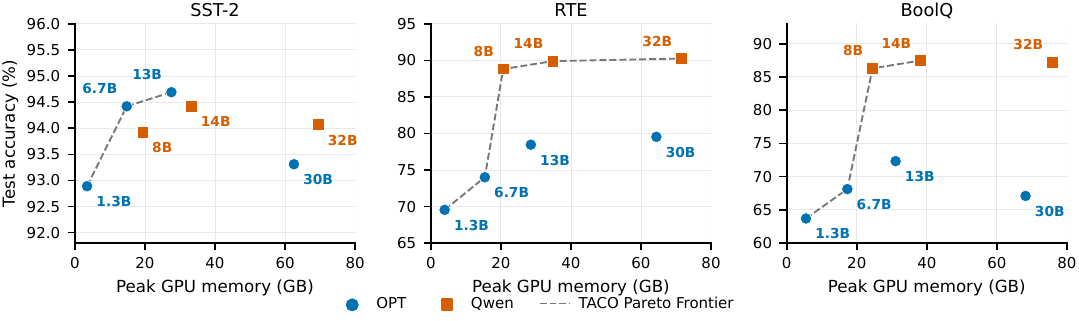}
    \vspace{-2mm}
    \caption{\small \textbf{Scaling behavior of TACO across the OPT and Qwen3 model families.} We plot the test performance against peak GPU memory as model size increases; see \Cref{tab:e2-scaling} in \Cref{sec:more_experiment_results} for full results.
    } 
    \label{fig:model_scaling}
    \vspace{-5mm}
\end{figure}











\smartparagraph{Transfer across architectures.} TACO is agnostic to model architecture. It achieves strong performance on Llama-3.1-8B and Qwen3-32B across all three tasks, while Pythia-12B and Mistral-24B exhibit larger task-dependent variation, particularly on RTE and BoolQ. Peak GPU memory remains close to the model footprint across architectures, consistent with TACO’s small persistent optimizer state.~\Cref{fig:architecture_table_opt_rte} reports architecture transfer results and examines learning rate transfer across OPT model scales, where all four model sizes attain their best validation accuracy at $3\times10^{-5}$.~The broader sweep in \Cref{sec:more_experiment_results} shows the same optimum across all evaluated model-dataset configurations, further supporting the use of a shared TACO configuration across model scales.

\smartparagraph{Gradient history and heavy hitter threshold.}~Practical TACO uses an EMA gradient history with sparse heavy-hitter. Comparing dense gradient history with $k\in\{64,32,16,8,4,2,1\}$, we find moderate sparse histories mostly retain performance comparable to dense gradient history while substantially reducing persistent optimizer state. In particular, $k=16$ uses only 38.3\,MB of optimizer state, a $34.4\times$ reduction from dense gradient history, while achieving the strongest BoolQ and RTE performance in the sweep. Hence, we use $k=16$ as the default heavy-hitter threshold. Across eight tasks, \(\mu=0.95\) improves six, including classification and multiple-choice tasks, with the largest gains on MultiRC (+8.9 F1A) and CB (+7.1\%); SQuAD and DROP are the two exceptions. We therefore retain \(\mu=0.95\) as default; see \Cref{tab:e3-history,tab:e3-history-k} in \Cref{sec:more_experiment_results} for full results.

\smartparagraph{Optimizer step budget.} TACO continues to improve beyond the default 2K-step budget. Increasing the budget from 2K to 16K steps improves BoolQ and RTE by +6.5\% and +4.3\%, respectively, and SQuAD by +2.1 F1. These results show that the default 2K-step budget does not represent TACO's performance ceiling; see
\Cref{tab:e3-steps} in \Cref{sec:more_experiment_results} for full results.

\smartparagraph{Update sparsity.}
We relax TACO's theoretically derived column-wise Top-1 update and evaluate denser Top-$k$ variants. Top-1 achieves the strongest final performance across all three evaluated tasks, while dense updates substantially degrade performance, particularly at Top-8 and beyond. These results support retaining Top-1 as the standard TACO update, consistent with the exact steepest-descent solution under the $1\!\to\!1$ geometry; see
\Cref{tab:e3-topk} in \Cref{sec:more_experiment_results} for full results.

\begin{figure*}[t]
\centering
\newlength{\panelwidth}
\newlength{\panelheight}
\setlength{\panelwidth}{0.494\linewidth}

\setlength{\panelheight}{0.691\panelwidth}

\begin{minipage}[t][\panelheight][c]{\panelwidth}
\centering
\scriptsize

\setlength{\tabcolsep}{2.5pt}

\resizebox{\linewidth}{!}{%
\begin{tabular}{@{}l cc cc cc@{}}
\toprule
\textbf{Model}
& \multicolumn{2}{c}{\textbf{SST-2}}
& \multicolumn{2}{c}{\textbf{RTE}}
& \multicolumn{2}{c}{\textbf{BoolQ}} \\
\cmidrule(lr){2-3}
\cmidrule(lr){4-5}
\cmidrule(lr){6-7}

&
\textbf{Acc.}
& \textbf{Mem.}
& \textbf{Acc.}
& \textbf{Mem.}
& \textbf{Acc.}
& \textbf{Mem.} \\[-1pt]

&
\scriptsize
& \scriptsize\textbf{(GB)}
& \scriptsize
& \scriptsize\textbf{(GB)}
& \scriptsize
& \scriptsize\textbf{(GB)} \\

\midrule

Pythia-12B
& 92.6 & 25.2
& 65.1 & 26.2
& 66.7 & 28.6 \\

Llama-3.1-8B
& 94.2 & 18.7
& 87.0 & 19.8
& 86.6 & 22.9 \\

Qwen3-32B
& 93.1 & 69.6
& 90.4 & 71.6
& 88.0 & 75.9 \\

Mistral-24B
& 87.8 & 50.5
& 54.4 & 52.0
& 69.5 & 55.1 \\

\bottomrule
\end{tabular}%
}\label{tab:e4-architectures}

\end{minipage}
\hfill
\begin{minipage}[t][\panelheight][c]{\panelwidth}
\centering
\vspace*{\fill}

\includegraphics[
  width=0.96\linewidth,
  keepaspectratio
]{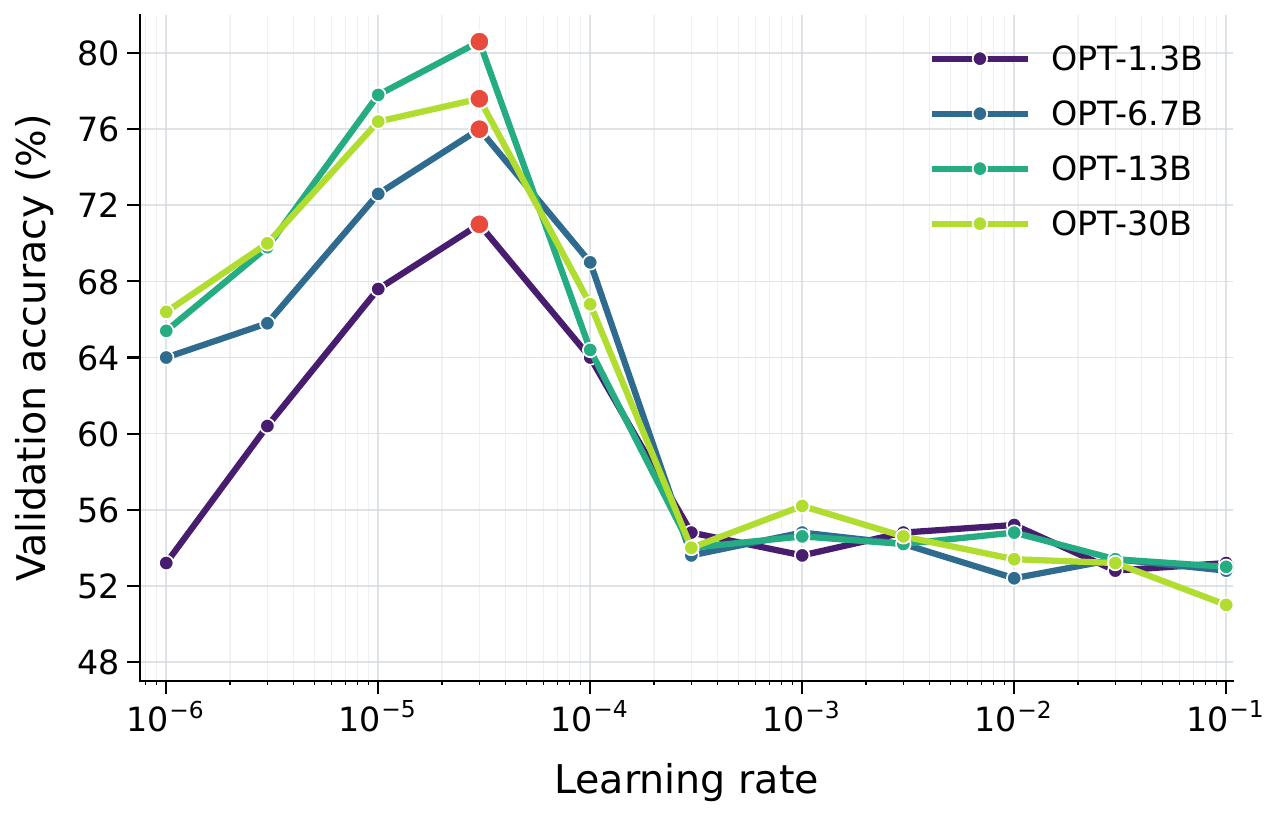}

\vspace*{\fill}
\end{minipage}
\vspace{-4mm}
\caption{\small{
\textbf{Transfer across model architectures and learning rate stability across model scales.}
\textit{Left:} TACO performance and peak GPU memory across Pythia-12B, Llama-3.1-8B, Qwen3-32B, and Mistral-24B on SST-2, RTE, and BoolQ, showing consistent memory efficiency with task-dependent performance. 
\textit{Right:} Learning rate sweep on RTE for OPT-1.3B, OPT-6.7B, OPT-13B, and OPT-30B. All four models attain their best validation accuracy at $3\times10^{-5}$, supporting learning rate transfer across model scales.
}}
\label{fig:architecture_table_opt_rte}
\vspace{-5mm}
\end{figure*}

\vspace{-2mm}
\section{Conclusion}
\vspace{-2mm}

We introduced the \textbf{T}ernary \textbf{A}bsolute-max \textbf{C}olumn-wise \textbf{O}ne-sparse (TACO) optimizer for memory-efficient full-parameter LLM fine-tuning. 
TACO is derived from the exact minimizer of a local optimization problem under a dimension-normalized $1\!\to\!1$ operator norm, yielding a ternary, column-wise one-sparse update. 
Exploiting this structure, practical TACO maintains a sparse, low-precision gradient history for matrix-valued parameters, making persistent optimizer-state memory nearly negligible. We establish a nonconvex convergence guarantee for vanilla TACO and, in a simplified fixed-input setting, show that it shares the continuation geometry and limiting solution of an Adam proxy while generally differing from Muon. 
We further establish $\mu$P spectral scaling under balanced winner counts and give a sufficient condition under which sparse history preserves the dense-history TACO update. 
Across model scales, architectures, and downstream tasks, TACO maintains competitive performance with a small memory footprint, enabling full-parameter fine-tuning of models up to 32B parameters on a single 80\,GB H100 GPU. More broadly, our results suggest that optimization geometry can provide a useful framework for designing memory-efficient first-order optimizers.

\subsection*{Acknowledgment} During this work, HanQin Cai was partially supported by the National Science Foundation Grant DMS-2603464, and Aritra Dutta was partially supported by the Florida Department of Health Grant AWD00007072.

\bibliography{iclr2027_conference}

@inproceedings{AdamW,
title={Decoupled Weight Decay Regularization},
author={Ilya Loshchilov and Frank Hutter},
booktitle={International Conference on Learning Representations},
year={2019}
}

@incollection{Pytorch,
title = {PyTorch: An Imperative Style, High-Performance Deep Learning Library},
author = {Paszke, Adam and Gross, Sam and Massa, Francisco and Lerer, Adam and Bradbury, James and Chanan, Gregory and Killeen, Trevor and Lin, Zeming and Gimelshein, Natalia and Antiga, Luca and Desmaison, Alban and Kopf, Andreas and Yang, Edward and DeVito, Zachary and Raison, Martin and Tejani, Alykhan and Chilamkurthy, Sasank and Steiner, Benoit and Fang, Lu and Bai, Junjie and Chintala, Soumith},
booktitle = {Advances in Neural Information Processing Systems 32},
pages = {8024--8035},
year = {2019},
publisher = {Curran Associates, Inc.}
}

@article{sahu2021rethinking,
  title={Rethinking gradient sparsification as total error minimization},
  author={Sahu, Atal and Dutta, Aritra and M Abdelmoniem, Ahmed and Banerjee, Trambak and Canini, Marco and Kalnis, Panos},
  journal={Advances in Neural Information Processing Systems},
  volume={34},
  pages={8133--8146},
  year={2021}
}

@inproceedings{dutta2020discrepancy,
  title={On the discrepancy between the theoretical analysis and practical implementations of compressed communication for distributed deep learning},
  author={Dutta, Aritra and Bergou, El Houcine and Abdelmoniem, Ahmed M and Ho, Chen-Yu and Sahu, Atal Narayan and Canini, Marco and Kalnis, Panos},
  booktitle={Proceedings of the AAAI Conference on Artificial Intelligence},
  volume={34},
  number={04},
  pages={3817--3824},
  year={2020}
}

@inproceedings{
malladi2023finetuning,
title={Fine-Tuning Language Models with Just Forward Passes},
author={Sadhika Malladi and Tianyu Gao and Eshaan Nichani and Alex Damian and Jason D. Lee and Danqi Chen and Sanjeev Arora},
booktitle={Thirty-seventh Conference on Neural Information Processing Systems},
year={2023}
}

@misc{lang2026poweringzerothordertrainingsubspace,
      title={Powering Up Zeroth-Order Training via Subspace Gradient Orthogonalization}, 
      author={Yicheng Lang and Changsheng Wang and Yihua Zhang and Mingyi Hong and Zheng Zhang and Wotao Yin and Sijia Liu},
      journal={arXiv preprint arXiv.2602.17155},
      year={2026}
}

@inproceedings{xu2021grace,
  title={Grace: A compressed communication framework for distributed machine learning},
  author={Xu, Hang and Ho, Chen-Yu and Abdelmoniem, Ahmed M and Dutta, Aritra and Bergou, El Houcine and Karatsenidis, Konstantinos and Canini, Marco and Kalnis, Panos},
  booktitle={2021 IEEE 41st international conference on distributed computing systems},
  pages={561--572},
  year={2021},
}

@inproceedings{liu2025sparse,
  author    = {Liu, Yong and Zhu, Zirui and Gong, Chaoyu and Cheng, Minhao and Hsieh, Cho-Jui and You, Yang},
  title     = {Sparse {MeZO}: Less Parameters for Better Performance in Zeroth-Order {LLM} Fine-Tuning},
  booktitle = {Advances in Neural Information Processing Systems (NeurIPS)},
  volume    = {38},
  pages     = {21139--21169},
  year      = {2025}
}

@inproceedings{adafactor,
  title={Adafactor: Adaptive learning rates with sublinear memory cost},
  author={Shazeer, Noam and Stern, Mitchell},
  booktitle={International conference on machine learning},
  pages={4596--4604},
  year={2018},
}

@inproceedings{came,
  title={Came: Confidence-guided adaptive memory efficient optimization},
  author={Luo, Yang and Ren, Xiaozhe and Zheng, Zangwei and Jiang, Zhuo and Jiang, Xin and You, Yang},
  booktitle={Proceedings of the 61st Annual Meeting of the Association for Computational Linguistics},
  pages={4442--4453},
  year={2023}
}

@article{adamw8bit,
  title={8-bit optimizers via block-wise quantization},
  author={Dettmers, Tim and Lewis, Mike and Shleifer, Sam and Zettlemoyer, Luke},
  journal={arXiv preprint arXiv:2110.02861},
  year={2021}
}

@InProceedings{galore,
  title = 	 {{G}a{L}ore: Memory-Efficient {LLM} Training by Gradient Low-Rank Projection},
  author =       {Zhao, Jiawei and Zhang, Zhenyu and Chen, Beidi and Wang, Zhangyang and Anandkumar, Anima and Tian, Yuandong},
  booktitle = 	 {Proceedings of the 41st International Conference on Machine Learning},
  pages = 	 {61121--61143},
  year = 	 {2024},
  volume = 	 {235},
  series = 	 {Proceedings of Machine Learning Research},
}

@article{wen2017terngrad,
  title={Terngrad: Ternary gradients to reduce communication in distributed deep learning},
  author={Wen, Wei and Xu, Cong and Yan, Feng and Wu, Chunpeng and Wang, Yandan and Chen, Yiran and Li, Hai},
  journal={Advances in neural information processing systems},
  volume={30},
  year={2017}
}

@article{signsgd,
  title={signSGD with majority vote is communication efficient and fault tolerant},
  author={Bernstein, Jeremy and Zhao, Jiawei and Azizzadenesheli, Kamyar and Anandkumar, Anima},
  journal={arXiv preprint arXiv:1810.05291},
  year={2018}
}

@article{signmuon,
  title={Signmuon: Communication-efficient distributed muon optimization},
  author={Mishra, Neel and Trivedi, Kushagara and Kumar, Pawan},
  journal={arXiv preprint arXiv:2605.16311},
  year={2026}
}

@article{qu2026can,
  title={Can Muon Fine-tune Adam-Pretrained Models?},
  author={Qu, Xingyu and Huang, Peigeng and Horvath, Samuel},
  journal={arXiv preprint arXiv:2605.10468},
  year={2026}
}

@article{bernstein2024oldoptimizer,
  title={Old Optimizer, New Norm: An Anthology},
  author={Bernstein, Jeremy and Newhouse, Laker},
  journal={arXiv preprint arXiv:2409.20325},
  year={2024}
}

@InProceedings{pethick2025normlmo,
  title={Training Deep Learning Models with Norm-Constrained {LMO}s},
  author={Pethick, Thomas and Xie, Wanyun and Antonakopoulos, Kimon and Zhu, Zhenyu and Silveti-Falls, Antonio and Cevher, Volkan},
  booktitle={Proceedings of the 42nd International Conference on Machine Learning},
  pages={49069--49104},
  year={2025},
  volume={267},
  series={Proceedings of Machine Learning Research}
}

@article{shah2025muon,
  title={Practical Efficiency of Muon for Pretraining},
  author={Shah, Ishaan and Polloreno, Anthony M. and Stratos, Karl and others},
  journal={arXiv preprint arXiv:2505.02222},
  year={2025}
}

@misc{bernstein2025deriving,
  author = {Jeremy Bernstein},
  title = {Deriving Muon},
  url = {https://jeremybernste.in/writing/deriving-muon},
  year = {2025}
}

@article{brown2020language,
  title={Language models are few-shot learners},
  author={Brown, Tom and Mann, Benjamin and Ryder, Nick and Subbiah, Melanie and Kaplan, Jared D and Dhariwal, Prafulla and Neelakantan, Arvind and Shyam, Pranav and Sastry, Girish and Askell, Amanda and others},
  journal={Advances in neural information processing systems},
  volume={33},
  pages={1877--1901},
  year={2020}
}

@inproceedings{houlsby2019parameter,
  title={Parameter-efficient transfer learning for NLP},
  author={Houlsby, Neil and Giurgiu, Andrei and Jastrzebski, Stanislaw and Morrone, Bruna and De Laroussilhe, Quentin and Gesmundo, Andrea and Attariyan, Mona and Gelly, Sylvain},
  booktitle={International conference on machine learning},
  pages={2790--2799},
  year={2019}
}

@article{hu2021lora,
  title={Lora: Low-rank adaptation of large language models},
  author={Hu, Edward J and Shen, Yelong and Wallis, Phillip and Allen-Zhu, Zeyuan and Li, Yuanzhi and Wang, Shean and Wang, Lu and Chen, Weizhu},
  journal={arXiv preprint arXiv:2106.09685},
  year={2021}
}

@inproceedings{zhao2025second,
 author = {Zhao, Yanjun and Dang, Sizhe and Ye, Haishan and Dai, Guang and Qian, Yi and Tsang, Ivor},
 booktitle = {International Conference on Learning Representations},
 editor = {Y. Yue and A. Garg and N. Peng and F. Sha and R. Yu},
 pages = {43496--43520},
 title = {Second-Order Fine-Tuning without Pain for LLMs: A Hessian Informed Zeroth-Order Optimizer},
 volume = {2025},
 year = {2025}
}

@inproceedings{yu2025subzero,
  title={Zeroth-Order Fine-Tuning of {LLM}s in Random Subspaces},
  author={Yu, Ziming and Zhou, Pan and Wang, Sike and Li, Jia and Tian, Mi and Huang, Hua},
  booktitle={Proceedings of the IEEE/CVF International Conference on Computer Vision},
  year={2025}
}

@article{micikevicius2022fp8,
  title={{FP8} Formats for Deep Learning},
  author={Micikevicius, Paulius and Stosic, Dusan and Burgess, Neil and Cornea, Marius and Dubey, Pradeep and Grisenthwaite, Richard and Ha, Sangwon and Heinecke, Alexander and Judd, Patrick and Kamalu, John and Mellempudi, Naveen and Oberman, Stuart and Shoeybi, Mohammad and Siu, Michael and Wu, Hao},
  journal={arXiv preprint arXiv:2209.05433},
  year={2022},
}

@inproceedings{gonzalezblalock2026flashoptim,
  title={{FlashOptim}: Optimizers for Memory-Efficient Training},
  author={Gonzalez Ortiz, Jose Javier and Gupta, Abhay and Rinard, Christopher and Blalock, Davis},
  booktitle={International Conference on Machine Learning},
  year={2026},
}

@inproceedings{zhu2025apollo,
  title={{APOLLO}: {SGD}-like Memory, {AdamW}-level Performance},
  author={Zhu, Hanqing and Zhang, Zhenyu and Cong, Wenyan and Liu, Xi and Park, Sem and Chandra, Vikas and Long, Bo and Pan, David Z. and Wang, Zhangyang and Lee, Jinwon},
  booktitle={Proceedings of Machine Learning and Systems},
  year={2025}
}

@misc{jordan2024muon,
  author={Jordan, Keller},
  title={Muon: An Optimizer for Hidden Layers in Neural Networks},
  howpublished={GitHub repository},
  year={2024}
}

@inproceedings{dagan2005pascal_rte1,
author = {Dagan, Ido and Glickman, Oren and Magnini, Bernardo},
title = {The {PASCAL} Recognising Textual Entailment Challenge},
year = {2005},
booktitle = {the First International Conference on Machine Learning Challenges: Evaluating Predictive Uncertainty Visual Object Classification, and Recognizing Textual Entailment}
}

@inproceedings{bar2006second,
   title={The second {PASCAL} recognising textual entailment challenge},
   author={Bar Haim, Roy and Dagan, Ido and Dolan, Bill and Ferro, Lisa and Giampiccolo, Danilo and Magnini, Bernardo and Szpektor, Idan},
   booktitle={Proceedings of the Second PASCAL Challenges Workshop on Recognising Textual Entailment},
   year={2006}
 }

@inproceedings{giampiccolo2007third_rte3,
    title = "The Third {PASCAL} Recognizing Textual Entailment Challenge",
    author = "Giampiccolo, Danilo  and
      Magnini, Bernardo  and
      Dagan, Ido  and
      Dolan, Bill",
    booktitle = "the {ACL}-{PASCAL} Workshop on Textual Entailment and Paraphrasing",
    year = "2007"
}

@inproceedings{bentivogli2009fifth_rte4,
  title={The Fifth {PASCAL} Recognizing Textual Entailment Challenge.},
  author={Bentivogli, Luisa and Clark, Peter and Dagan, Ido and Giampiccolo, Danilo},
  booktitle={Proceedings of the Text Analysis Conference},
  year={2009}
}

@inproceedings{socher2013recursive_sst-2,
    title = "Recursive Deep Models for Semantic Compositionality Over a Sentiment Treebank",
    author = "Socher, Richard  and
      Perelygin, Alex  and
      Wu, Jean  and
      Chuang, Jason  and
      Manning, Christopher D.  and
      Ng, Andrew  and
      Potts, Christopher",
    booktitle = "Proceedings of the 2013 Conference on Empirical Methods in Natural Language Processing",
    year = "2013"
}

@article{zhang2022opt,
  title={Opt: Open pre-trained transformer language models},
  author={Zhang, Susan and Roller, Stephen and Goyal, Naman and Artetxe, Mikel and Chen, Moya and Chen, Shuohui and Dewan, Christopher and Diab, Mona and Li, Xian and Lin, Xi Victoria and others},
  journal={arXiv preprint arXiv:2205.01068},
  year={2022}
}

@inproceedings{wang2019superglue,
  title={Superglue: A stickier benchmark for general-purpose language understanding systems},
  author={Wang, Alex and Pruksachatkun, Yada and Nangia, Nikita and Singh, Amanpreet and Michael, Julian and Hill, Felix and Levy, Omer and Bowman, Samuel},
  booktitle={Advances in neural information processing systems},
  volume={32},
  year={2019}
}

@inproceedings{rajpurkar-etal-2016-squad,
    title = "{SQ}u{AD}: 100,000+ Questions for Machine Comprehension of Text",
    author = "Rajpurkar, Pranav  and
      Zhang, Jian  and
      Lopyrev, Konstantin  and
      Liang, Percy",
    booktitle = "Proceedings of the 2016 Conference on Empirical Methods in Natural Language Processing",
    year = "2016",
    pages = "2383--2392"
}

@inproceedings{dua-etal-2019-drop,
    title = "{DROP}: A Reading Comprehension Benchmark Requiring Discrete Reasoning Over Paragraphs",
    author = "Dua, Dheeru  and
      Wang, Yizhong  and
      Dasigi, Pradeep  and
      Stanovsky, Gabriel  and
      Singh, Sameer  and
      Gardner, Matt",
    booktitle = "Proceedings of the 2019 Conference of the North {A}merican Chapter of the Association for Computational Linguistics: Human Language Technologies, Volume 1",
    year = "2019",
    pages = "2368--2378"
}

@inproceedings{clark-etal-2019-boolq,
    title = "{B}ool{Q}: Exploring the Surprising Difficulty of Natural Yes/No Questions",
    author = "Clark, Christopher  and
      Lee, Kenton  and
      Chang, Ming-Wei  and
      Kwiatkowski, Tom  and
      Collins, Michael  and
      Toutanova, Kristina",
    booktitle = "Proceedings of the 2019 Conference of the North {A}merican Chapter of the Association for Computational Linguistics: Human Language Technologies",
    year = "2019",
    pages = "2924--2936",
}

@inproceedings{de2019commitmentbank,
  title={The commitmentbank: Investigating projection in naturally occurring discourse},
  author={De Marneffe, Marie-Catherine and Simons, Mandy and Tonhauser, Judith},
  booktitle={Sinn und Bedeutung},
  volume={23},
  number={2},
  pages={107--124},
  year={2019}
}

@inproceedings{roemmele2011choice,
  title={Choice of Plausible Alternatives: An Evaluation of Commonsense Causal Reasoning.},
  author={Roemmele, Melissa and Bejan, Cosmin Adrian and Gordon, Andrew S},
  year={2011}
}

@inproceedings{khashabi2018looking,
  title={Looking beyond the surface: A challenge set for reading comprehension over multiple sentences},
  author={Khashabi, Daniel and Chaturvedi, Snigdha and Roth, Michael and Upadhyay, Shyam and Roth, Dan},
  booktitle={Proceedings of the 2018 Conference of the North American Chapter of the Association for Computational Linguistics: Human Language Technologies},
  pages={252--262},
  year={2018}
}

@article{qwen3,
    title={Qwen3 Technical Report}, 
    author={An Yang and Anfeng Li and Baosong Yang and Beichen Zhang and Binyuan Hui and Bo Zheng and Bowen Yu and Chang Gao and Chengen Huang and Chenxu Lv and Chujie Zheng and Dayiheng Liu and Fan Zhou and Fei Huang and Feng Hu and Hao Ge and Haoran Wei and Huan Lin and Jialong Tang and Jian Yang and Jianhong Tu and Jianwei Zhang and Jianxin Yang and Jiaxi Yang and Jing Zhou and Jingren Zhou and Junyang Lin and Kai Dang and Keqin Bao and Kexin Yang and Le Yu and Lianghao Deng and Mei Li and Mingfeng Xue and Mingze Li and Pei Zhang and Peng Wang and Qin Zhu and Rui Men and Ruize Gao and Shixuan Liu and Shuang Luo and Tianhao Li and Tianyi Tang and Wenbiao Yin and Xingzhang Ren and Xinyu Wang and Xinyu Zhang and Xuancheng Ren and Yang Fan and Yang Su and Yichang Zhang and Yinger Zhang and Yu Wan and Yuqiong Liu and Zekun Wang and Zeyu Cui and Zhenru Zhang and Zhipeng Zhou and Zihan Qiu},
    journal = {arXiv preprint arXiv:2505.09388},
    year={2025}
}

@inproceedings{biderman2023pythia,
  title={Pythia: A suite for analyzing large language models across training and scaling},
  author={Biderman, Stella and Schoelkopf, Hailey and Anthony, Quentin Gregory and Bradley, Herbie and O’Brien, Kyle and Hallahan, Eric and Khan, Mohammad Aflah and Purohit, Shivanshu and Prashanth, USVSN Sai and Raff, Edward and others},
  booktitle={International Conference on Machine Learning},
  pages={2397--2430},
  year={2023}
}

@article{llama3herd2024,
  title={The Llama 3 Herd of Models},
  author={Dubey, Abhimanyu and Jauhri, Abhinav and Pandey, Abhinav and Kadian, Agam and Al-Dahle, Ahmad and Letman, Aiesha and Akhoundi, Akhil and Alibek, Akbar and Alumae, Tanel and Anichini, Anirudh \textit{et al.}},
  journal={arXiv preprint arXiv:2407.21783},
  year={2024},
}

@misc{MISTRAL-CITATION,
  author       = {{Mistral AI Team}},
  title        = {Mistral-Small-24B-Base-2501 Foundation Model},
  year         = {2025},
  publisher    = {Hugging Face},
  howpublished = {\url{https://huggingface.co/mistralai/Mistral-Small-24B-Base-2501}},
  note         = {Released January 2025}
}

@article{STP,
   title = {{Stochastic Three Points Method for Unconstrained Smooth Minimization}} ,
   author =    {Bergou, El Houcine and Gorbunov, Eduard and Richt\'{a}rik, Peter },
   journal = {SIAM Journal on Optimization },
   year =      {2020},
   address = {},
   series =    {},
   edition =   {},
   volume =    {30(4):2726–2749},
   url =       {}}

@inproceedings{SMTP,
  title     = {A Stochastic Derivative Free Optimization Method with Momentum},
  author    = {Gorbunov, Eduard and Bibi, Adel and Sener, Ozan and Bergou, El Houcine and Richtarik, Peter},
  booktitle = {International Conference on Learning Representations},
  year      = {2020}
}

@article{IS-STP,
   author =    { Bibi, Adel and  Bergou, El Houcine and Sener, Ozan and  Ghanem, Bernard and  Richt\'{a}rik, Peter },
   title = {{ A stochastic derivative-free optimization
method with importance sampling: Theory and Learning to Control}} ,
   journal = {Proceedings of the 34th AAAI Conference on Artificial Intelligence},
   year =      {2020},
   address = {New York, NY},
   series =    {},
   edition =   {},
   volume =    {},
   number={},
   pages={ 3275-3282},
   url =       {}}

@article{MiSTP, 
title={{Minibatch Stochastic Three Points Method for Unconstrained Smooth Minimization}}, volume={38},  number={18}, journal={Proceedings of the AAAI Conference on Artificial Intelligence}, author={Boucherouite, Soumia and Malinovsky, Grigory and Richtárik, Peter and Bergou, El Houcine}, year={2024}, pages={20344-20352} }

@InProceedings{cai2021zobcd,
  title     = {A Zeroth-Order Block Coordinate Descent Algorithm for Huge-Scale Black-Box Optimization},
  author    = {Cai, HanQin and Lou, Yuchen and Mckenzie, Daniel and Yin, Wotao},
  booktitle = {Proceedings of the 38th International Conference on Machine Learning},
  pages     = {1193--1203},
  year      = {2021},
  volume    = {139}
}

@article{kim2025cars,
  title     = {Curvature-Aware Derivative-Free Optimization},
  author    = {Kim, Bumsu and McKenzie, Daniel and Cai, HanQin and Yin, Wotao},
  journal   = {Journal of Scientific Computing},
  year      = {2025},
  volume    = {103},
  number    = {43},
  pages     = {1--28}
}

@article{lv2023full,
  title={Full Parameter Fine-tuning for Large Language Models with Limited Resources},
  author={Lv, Kai and Yang, Yuqing and Liu, Tengxiao and Gao, Qinghui and Guo, Qipeng and Qiu, Xipeng},
  journal={arXiv preprint arXiv:2306.09782},
  year={2023}
}

@article{lv2023adalomo,
  title={AdaLomo: Low-memory Optimization with Adaptive Learning Rate},
  author={Lv, Kai and Yan, Hang and Guo, Qipeng and Lv, Haijun and Qiu, Xipeng},
  journal={arXiv preprint arXiv:2310.10195},
  year={2023}
}
\bibliographystyle{iclr2027_conference}

\clearpage
\begin{center}
\Large
Supplementary Material for\\
\textsc{TACO:~{T}ernary {A}bsolute-max {C}olumn-wise {O}ne-sparse\\ Optimizer for LLM Fine-Tuning}
\end{center}
\tableofcontents

\vspace{0.2in}

\appendix
\crefalias{section}{appendix}
\crefname{appendix}{appendix}{appendices}
\Crefname{appendix}{Appendix}{Appendices}

\textbf{Organization of the Appendix.} We organize the Appendix as follows.
\Cref{sec:theory-extension} contains the proofs of the theorems in \Cref{sec:TACO,sec:geometry}. In particular, it contains the proofs of \Cref{thm:taco-convergence} to \Cref{thm:taco-exact-history}.
\Cref{sec:experiment_setup} describes the experimental setup, including the models and datasets, implementation details, hyperparameter settings, and hardware configuration.
\Cref{sec:more_experiment_results} also provides additional experimental results, including the complete optimizer comparison, model scaling results, ablation studies, and architecture transfer results.


\section{Theoretical Analysis and Proofs}
\label{sec:theory-extension}
This section provides proofs and additional theoretical context for the results in \Cref{sec:TACO,sec:geometry}. We first establish the convergence guarantee for vanilla TACO, then examine its continuation geometry and $\mu$P scaling, and conclude with a condition under which sparse gradient histories preserve the dense-history TACO update.

\subsection{Vanilla TACO Convergence}\label{appendix: Taco convergence}

We begin by proving the convergence guarantee in Theorem~\ref{thm:taco-convergence} for the history-free iteration. The proof uses $\langle G_t, D_t\rangle_F=\|G_t\|_{(1,n\to1,m)^*}$ to express the first-order decrease exactly, while smoothness contributes the step-size penalty $L\eta_t^2/2$.

\setcounter{theorem}{0}
\begin{theorem}{(Nonconvex convergence of vanilla TACO)}
Let Assumption~\ref{TACO:L-smooth_assumption} hold and define $\Delta_F=F(W_0)-F_{\star}$. The iterates $\{W_t\}_{t=0}^{T}$ in \eqref{eq:vanilla TACO update} with step size $\eta_t>0$ follow
\begin{equation*}
    \min_{0\le t<T}\|\nabla F(W_t)\|_{(1,n\to1,m)^*}
    \le
    \frac{\Delta_F+\frac{L}{2}\sum_{t=0}^{T-1}\eta_t^2}
         {\sum_{t=0}^{T-1}\eta_t}.
\end{equation*}
In particular, with $\eta=\sqrt{\tfrac{2\Delta_F}{LT}}$, we have
    $
        \min_{0\le t<T} \|\nabla F(W_t)\|_{(1,n \to 1,m)^*}\le \sqrt{\frac{2L\Delta_F}{T}}.
    $
\end{theorem}
\begin{proof}
    By Assumption \ref{TACO:L-smooth_assumption}, using the $L$-smoothness of $F$, for any update direction $D$, we have
    \begin{equation}\label{eq:desecent}
        F(W + D) \le F(W) + \lara{\nabla F(W), D}_F + \frac{L}{2} \|D\|_{(1,n\to 1,m)}^2.
    \end{equation}
    Since $\nabla F(W)=G,$ with $W_{t+1} = W_t - \eta_t D_t$, we have $\|D_t\|_{(1,n \to 1, m)} \le 1$, $\lara{G_t, D_t}_F = \|G_t\|_{(1,n\to 1,m)^*}$, and (\ref{eq:desecent}) reduces to
    \begin{equation}\label{eq:desecent1}
        F(W_{t+1}) \le F(W_t) - \eta_t \| G_t\|_{(1,n \to 1,m)^*} + \frac{L}{2} \eta_t^2.
    \end{equation}
    Summing (\ref{eq:desecent1}) with $t$ varying from $0$ to $T-1$ gives
    \begin{equation}
        F(W_T) - F(W_0) \le - \sum_{t=0}^{T-1} \eta_t \| G_t\|_{(1,n \to 1,m)^*} + \frac{L}{2} \sum_{t=0}^{T-1} \eta_t^2.
    \end{equation}
    Note that $F \ge F_{\star}$. Reordering the inequality, we obtain
    \begin{equation}
        \parentheses{\min_{t< T}\| G_t\|_{(1,n \to 1,m)^*}} \sum_{t=0}^{T-1} \eta_t  \le \Delta F + \frac{L}{2}\sum_{t=0}^{T-1} \eta_t^2,
    \end{equation}
    which proves the result.
\end{proof}

The resulting bound balances the initial objective gap $\Delta_F$ against the accumulated step-size penalty. Choosing $\eta=\sqrt{2\Delta_F/(LT)}$ gives the $T^{-1/2}$ stationarity bound.

{

\begin{theorem}
    {(Nonconvex convergence of vanilla TACO with stochastic gradients)}
    Let \Cref{TACO:L-smooth_assumption,rebuttal:assumption:unbias_estimate,rebuttal:assumption:bounded_variance} hold. For any initial $W_0$, fix $T$ iterations with constant learning rate $\eta > 0$ and minibatch size $B$, and let $\{W_t\}$ be generated by \eqref{eq:vanilla TACO update} with stochastic gradients $D_t = \frac{m}{n}\mathcal{S}(\tilde G_t)$. Denote $\Delta_F = F(W_0) - F_\star$. Then
    \begin{equation}
        \min_{0\le t< T} \E \|\nabla F(W_t)\|_{(1,n\to 1,m)^*} \le \frac{\Delta_F}{\eta T} + \frac{2m\nu}{\sqrt{nB}} + \frac{\eta L}{2}.
    \end{equation}
    In particular, choosing $\eta = \sqrt{\frac{2\Delta_F}{LT}}$, setting $T\ge \frac{8L\Delta_F}{\eps^2}$ and $B\ge \frac{16m^2\nu^2}{n\eps^2}$ guarantees 
    \begin{equation}
        \min_{0\le t <T} \E\|\nabla F(W_t)\|_{(1,n\to 1,m)^*}\le \eps.
    \end{equation}
\end{theorem}

The proof is based on the following two lemmas (writing $\|\cdot\|_* := \|\cdot\|_{(1,n\to 1,m)^*}$ and $D(X) := \frac{m}{n}\mathcal{S}(X)$).

\begin{lemma}\label{rebuttal:lemma1}
    For any $G, \tilde G\in \R^{m\times n}$, we have
    \begin{equation}
        \lara{G, D(\tilde G)}_F \ge \|G\|_* - 2\|\tilde G - G\|_*.
    \end{equation}
\end{lemma}
\begin{proof}
    \begin{equation}
        \begin{split}
            \lara{G, D(\tilde G)}_F & = \lara{\tilde G, D(\tilde G)}_F + \lara{G - \tilde G, D(\tilde G)}_F\\
            &\ge \|\tilde G\|_* - \|G - \tilde G\|_*\\
            &\ge \|G\|_* - 2 \|G - \tilde G\|_*.
        \end{split}
    \end{equation}
\end{proof}

\begin{lemma}\label{rebuttal:lemma2}
    Under \Cref{rebuttal:assumption:unbias_estimate,rebuttal:assumption:bounded_variance}, we have
    \begin{equation}
        \E_t \|\tilde G_t - G_t\|_* \le \frac{m\nu}{\sqrt{nB}}.
    \end{equation}
\end{lemma}
\begin{proof}
    We have
    \begin{equation}
        \|\tilde G_t - G_t\|_* = \frac{m}{n}\sum_{j=1}^n \| (\tilde G_t -G_t)_{:, j}\|_\infty \le \frac{m}{n} \sum_{j=1}^n \| (\tilde G_t -G_t)_{:, j}\|_2 \le \frac{m}{\sqrt{n}} \|\tilde G_t - G_t\|_F.
    \end{equation}
    Taking expectations and applying Jensen's inequality gives
    \begin{equation}
        \E_t \|\tilde G_t -G_t\|_* \le \frac{m}{\sqrt{n}}\sqrt{\E_t \|\tilde G_t - G_t\|_F^2} \le \frac{m\nu}{\sqrt{nB}}.
    \end{equation}
\end{proof}

We are ready to present the proof of \Cref{rebuttal:thm:stochastic gradients vannila TACO}.

\begin{proof}[Proof of \Cref{rebuttal:thm:stochastic gradients vannila TACO}]
    By \Cref{TACO:L-smooth_assumption},
    \begin{equation}
    \begin{aligned}
        F(W_{t+1}) &\le F(W_t) - \eta\lara{G_t, D(\tilde G_t)}_F + \frac{L \eta^2}{2} \|D(\tilde G_t)\|^2\\
        &\le F(W_t) - \eta\lara{G_t, D(\tilde G_t)}_F + \frac{L \eta^2}{2}.
    \end{aligned}
    \end{equation}
    Taking conditional expectations and applying \Cref{rebuttal:lemma1},
    \begin{equation}
        \E_t [F(W_{t+1})] \le F(W_t) - \eta \|G_t\|_* + 2\eta \E_t\|\tilde G_t- G_t\|_* + \frac{L \eta^2}{2}.
    \end{equation}
    Applying \Cref{rebuttal:lemma2}, summing over $0\le t< T$, and reordering the inequality yields
    \begin{equation}
        \eta T \min_{0\le t< T} \E\|G_t\|_* \le \eta \sum_{t=0}^{T-1} \E \|G_t\|_* \le \Delta_F + \frac{2m\nu\eta T}{\sqrt{nB}} + \frac{LT\eta^2}{2}.
    \end{equation}
    Dividing by $\eta T$ yields the desired inequality.
\end{proof}

}

{


}

\subsection{Multilayer analysis}
For neural network training, consider matrix parameters
\begin{equation*}
    W = \parentheses{W^{(1)}, \cdots, W^{(L)}}, W^{(\ell)} \in \R^{m_\ell \times n_\ell},
\end{equation*}
and the objective $\mathcal{L}: \R^{m_1\times n_1}\times \cdots \times\R^{m_L \times n_L} \to \R$. For each layer, define $N_\ell (D^{(\ell)}):= \|D^{(\ell)}\|_{(1,n_\ell \to 1,m_\ell)}$. On the product space $\R^{m_1\times n_1}\times \cdots \times\R^{m_L \times n_L}$, define $\mathcal{N}(D) = \max_\ell \| D^{(\ell)}\| _{(1,n_\ell \to 1,m_\ell)}$.
The dual norm of $\mathcal{N}$ is given by
\begin{equation*}
    \mathcal{N}^* (G) = \sum_{\ell = 1}^L \|G^{(\ell)}\|_{(1,n_\ell \to 1,m_\ell)^*} = \sum_{\ell = 1}^L \frac{m_\ell}{n_\ell} \sum_{j} \|g_j^{(\ell)}\|_\infty.
\end{equation*}
The local minimization problem under this multilayer setting is
\begin{equation*}
    \min_{D} \sum_{\ell = 1}^L \lara{G^{(\ell)}, D^{(\ell)}}_F, \quad \text{s.t. } \mathcal{N}(D) \le 1.
\end{equation*}
The exact minimizer is obtained simultaneously in every layer,
\begin{equation*}
    D^{(\ell)} = \frac{m_\ell}{n_\ell} S(G^{(\ell)}), \quad \ell = 1, \cdots, L.
\end{equation*}
Thus, the parameter update is
\[
    W_{t+1}^{(\ell)} = W_{t}^{(\ell)} - \eta_t D_t^{(\ell)},\quad \ell = 1, \cdots, L.
\]
Note that every layer receives its geometry scaling factor $\frac{m_\ell}{n_\ell}$, and the optimization learning rate $\eta_t$ remains a single shared scalar across all layers. This distinction is important; it suggests that the learning rate is an optimization hyperparameter that is invariant to the parameter shape.

\subsection{Continuation Geometry, Implicit Bias, and $\mu$P-scaling Analysis}

Our implicit-bias analysis extends the implicit-bias study of \cite{qu2026can} beyond their least-squares setting. Rather than restricting the objective to a squared loss, we consider the more general class $F(W)=\phi(Wz)$ for a differentiable function $\phi$, and characterize the continuation geometry and limiting solutions induced by TACO, the Adam proxy, and the Muon proxy under a common target output. The least-squares problem studied in \cite{qu2026can} is recovered as a specific instance of this framework; we return to this case after the general results and show how our solutions reduce to their previously derived SignGD and Muon endpoints.

We consider $F(W)=\phi(Wz)$ and ask which weights are selected among matrices producing the same target output $u^\star$. For the implicit-bias analysis, we use the following optimizer proxies:
\begin{tcolorbox}[
    colback=gray!5,
    colframe=gray!75!black,
    boxrule=0.8pt,
    arc=2mm,
    left=2pt,
    right=2pt,
    top=1pt,
    bottom=1pt
]

The \textbf{Adam proxy} is initialized with $M_{-1}=V_{-1}=0$ and follows
\begin{equation}\label{eq:adam-proxy}
\begin{array}{ccc}
\displaystyle M_t=\beta_1M_{t-1}+(1-\beta_1)G_t
&
\displaystyle V_t=\beta_2V_{t-1}+(1-\beta_2)G_t\odot G_t
&
\displaystyle \hat M_t=\frac{M_t}{1-\beta_1^{t+1}}
\\[3pt]
\multicolumn{3}{c}{
\displaystyle
\hat V_t=\frac{V_t}{1-\beta_2^{t+1}}
\qquad
W_{t+1}=W_t-\eta_t\frac{\hat M_t}{\sqrt{\hat V_t}}
}
\end{array}
\end{equation}
\end{tcolorbox}

\begin{tcolorbox}[
    colback=gray!5,
    colframe=gray!75!black,
    boxrule=0.8pt,
    arc=2mm,
    left=2pt,
    right=2pt,
    top=1pt,
    bottom=1pt
]

The \textbf{Muon proxy} is initialized with $B_{-1}=0$ and follows
\begin{equation}\label{eq:muon-proxy}
\begin{aligned}
B_t = \beta B_{t-1}+G_t,\quad
W_{t+1} = W_t-\eta_t\mathcal{P}(B_t), \quad
\mathcal{P}(B_t) = UV^\top,
\quad
B_t=U\Sigma V^\top.
\end{aligned}
\end{equation}
Here, $G_t=\nabla F(W_t)$.
\end{tcolorbox}



We note that both proxies use momentum; however, their resulting theoretical developments are invariant to momentum. We use momentum for a more realistic case.

The proofs of Theorems~\ref{thm:adam-taco} and~\ref{thm:muon-sep} begin with $G_t=g_tz^\top$, where $g_t=\nabla\phi(W_tz)$. As a result, the  gradient columns are scaled copies of the same vector. We track this structure through the optimizer histories where TACO's and Adam proxy's winners yield the factor $\operatorname{sign}(z)^\top$, whereas the Muon proxy obtains $z^\top$.

\setcounter{theorem}{2}
\begin{theorem}({Adam and TACO continuation equivalence})
Let \Cref{ass:coherent,ass:attained} hold, and let the Adam Proxy follow \eqref{eq:adam-proxy}. Then both the TACO and Adam sequences lie in $\mathcal{C}(W_0,\operatorname{sign}(z))$ for all $t\geq 0$, and
\[W_{\rm TACO}^\star=W_{\rm Adam}^\star=W_0+\frac{(u^\star-W_0z)\operatorname{sign}(z)^\top}{\|z\|_1}.\]
\end{theorem}
\vspace{-6mm}
\begin{proof}
Assumption~\ref{ass:coherent} and the chain rule give $G_t=\nabla F(W_t)=g_tz^\top$ with $g_t=\nabla\phi(W_tz)$.

For TACO,  $M_t=p_tz^\top$ for some $p_t\in\mathbb{R}^m$. This holds at $t=0$ since $M_{-1}=0$. If $M_{t-1}=p_{t-1}z^\top$, then $\mu M_{t-1}+(1-\mu)G_t=\tilde{p}_tz^\top$ with $\tilde{p}_t=\mu p_{t-1}+(1-\mu)g_t$. 

Column $j$ of $\tilde{p}_tz^\top$ equals $z_j\tilde{p}_t$, so its entry magnitudes are $|z_j||\tilde{p}_{t,i}|$ and their ordering does not depend on $j$.

The nonzero columns have the same magnitude ordering, with
equal-magnitude entries ordered by row index. Hence, $H_k$
retains the same row set $\mathcal I_t$ in each nonzero column,
while zero columns remain zero, giving $M_t=p_tz^\top$.

The same ordering argument applied to $S$ selects a single row $i_t^\star$ in every column with $z_j\neq 0$ and returns a zero column whenever $z_j=0$, so $S(M_t)=\operatorname{sign}(p_{t,i_t^\star})e_{i_t^\star}\operatorname{sign}(z)^\top$ and $D_t=\frac{m}{n}S(M_t)$ has the form $a\operatorname{sign}(z)^\top$.

For the Adam proxy,  $M_t=p_tz^\top$ and $V_t=q_t(z\odot z)^\top$ with $q_t\geq 0$ entrywise. Both hold at $t=0$ and propagate, since $G_t\odot G_t=(g_t\odot g_t)(z\odot z)^\top$. Bias correction preserves the two forms, so the $(i,j)$ entry of $\hat{M}_t/\sqrt{\hat{V}_t}$ equals $c_t(p_{t,i}/\sqrt{q_{t,i}})\operatorname{sign}(z_j)$ with $c_t=\sqrt{1-\beta_2^{t+1}}/(1-\beta_1^{t+1})$, giving an update of the form $a\operatorname{sign}(z)^\top$ as well.

Hence, both sequences lie in $\mathcal{C}(W_0,\operatorname{sign}(z))$. 

For $W=W_0+a\operatorname{sign}(z)^\top$ we have $\operatorname{sign}(z)^\top z=\|z\|_1$, hence $Wz=W_0z+a\|z\|_1$. Imposing $Wz=u^\star$ and using $\|z\|_1>0$ determines $a=(u^\star-W_0z)/\|z\|_1$.
\end{proof}

\begin{theorem}({Muon continuation separation})
Let \Cref{ass:coherent,ass:attained} hold and let the Muon proxy follow \eqref{eq:muon-proxy}. Further let $r=u^\star-W_0z$. Then Muon's iterates lie in $\mathcal{C}(W_0,z)$ and
$$W_{\rm Muon}^\star=W_0+\frac{rz^\top}{\|z\|_2^2},\qquad W_{\rm Muon}^\star-W_{\rm TACO}^\star=r\left(\frac{z}{\|z\|_2^2}-\frac{\operatorname{sign}(z)}{\|z\|_1}\right)^\top.$$
\end{theorem}
\vspace{-6mm}
\begin{proof}
The gradient can be represented via $G_t=g_tz^\top$, and $B_t=b_tz^\top$ by induction from $B_{-1}=0$. 

If $b_t\neq 0$, then
$$ b_tz^\top=(b_t/\|b_t\|_2)(\|b_t\|_2\|z\|_2)(z/\|z\|_2)^\top $$
is a singular value decomposition, so $\mathcal{P}(B_t)=b_tz^\top/(\|b_t\|_2\|z\|_2)$.

Hence, every update has the form $az^\top$, so the iterates and their limit lie in the closed set $\mathcal{C}(W_0,z)$.

For $W=W_0+az^\top$ we have $Wz=W_0z+a\|z\|_2^2$, and $Wz=u^\star$ forces $a=r/\|z\|_2^2$, which gives the endpoint. 

Subtracting the endpoint of \Cref{thm:adam-taco} gives the displayed difference, and $z^\top z/\|z\|_2^2=\operatorname{sign}(z)^\top z/\|z\|_1=1$ gives $(W_{\rm Muon}^\star-W_{\rm TACO}^\star)z=0$.
\end{proof}
\vspace{-4mm}
The common target fixes the output correction $r=u^\star-W_0z$, but the continuation set determines how this correction is represented in the weight space. We illustrate this distinction next through least squares, connecting our formulas to the optimizer mismatch analysis of \citet{qu2026can}.

\begin{remark}\label{rmk:layerwise}{(Extension to deep neural networks)}
\Cref{thm:adam-taco} extends directly to deep networks under sequential layerwise fine-tuning. Consider
$$
h_0=x,\qquad
h_\ell=\sigma_\ell(W^{(\ell)} h_{\ell-1}),
\qquad \ell=1,\ldots,L,
$$
and optimize one $W^{(\ell)}$ at a time while keeping all other layers fixed. During each stage, $h_{\ell-1}$ is fixed, so the objective depends on $W^{(\ell)}$ only through $W^{(\ell)} h_{\ell-1}$ and therefore satisfies \Cref{ass:coherent} with $z=h_{\ell-1}$. \Cref{thm:adam-taco} then implies that TACO and the Adam proxy reach the same endpoint for that layer. Since the resulting networks are identical, they produce the same input to the next layer, and repeating this argument over $\ell=1,\ldots,L$ yields identical final networks.
\end{remark}

Theorems~\ref{thm:adam-taco} and \ref{thm:muon-sep} prove the general case. Below, we illustrate Theorems~\ref{thm:adam-taco} and~\ref{thm:muon-sep} through a specific example.

\smartparagraph{Least-squares example.}
 Consider the least-squares objective function
$$
F_{\mathrm{LS}}(W)=\frac12\|Wz-u^\star\|_2^2,
$$
where $W\in\mathbb{R}^{m\times n}$, $z\in\mathbb{R}^n$ is nonzero, and $u^\star\in\mathbb{R}^m$. This satisfies Assumption~\ref{ass:coherent} with $\phi(u)=\frac12\|u-u^\star\|_2^2$. Let $r=u^\star-W_0z$. Proposition~D.4 of \citet{qu2026can} gives the solutions
$$
W_{\rm SignGD}^\star=W_0+\frac{r\operatorname{sign}(z)^\top}{\|z\|_1}, \qquad 
W_{\rm Muon}^\star
=W_0+\frac{rz^\top}{\|z\|_2^2}
$$
for SignGD and momentum-free Muon with exact orthogonalization from arbitrary initialization $W_0$. Their convergence result uses positive step sizes satisfying $\eta_t\to0$ and $\sum_{t=0}^{\infty}\eta_t=\infty$.

Under Assumption~\ref{ass:attained}, Theorem~\ref{thm:adam-taco} gives
$$
W_{\rm TACO}^\star=W_{\rm Adam}^\star=W_{\rm SignGD}^\star,
$$
while Theorem~\ref{thm:muon-sep} gives the same Muon endpoint for the momentum-based proxy. Thus, TACO and the Adam proxy select the same limiting weights as SignGD in this setting.

These endpoints also characterize different minimum-norm corrections, where TACO and the Adam proxy inherit the minimum-max-norm correction identified for SignGD, whereas Muon selects a minimum-spectral-norm correction. For nonzero initialization, these statements concern the displacement $W-W_0$, rather than the norm of the final weights themselves.

The distinction between the continuation sets comes from their fixed right factors. In $\mathcal{C}(W_0,w)$, the coefficient $a\in\mathbb{R}^m$ varies, while $w\in\mathbb{R}^n$ is fixed. Every row of $W-W_0=aw^\top$ is therefore a scalar multiple of $w^\top$. Changing $a$ changes these row coefficients but cannot change their direction. Both $\mathcal{C}(W_0,z)$ and $\mathcal{C}(W_0,\operatorname{sign}(z))$ have dimension $m$, but they coincide only when $z=c\operatorname{sign}(z)$ for some $c>0$. A common nonzero correction
$$
az^\top=b\operatorname{sign}(z)^\top
$$
has a nonzero row satisfying $a_i z=b_i\operatorname{sign}(z)$, which forces this proportionality. Otherwise,
$$
\mathcal{C}(W_0,z)\cap\mathcal{C}(W_0,\operatorname{sign}(z))
=\{W_0\},
$$
and the two endpoints differ whenever $r\neq0$.

The two factors are not interchangeable. In particular, $z\operatorname{sign}(z)^\top$ and $\operatorname{sign}(z)z^\top$ are both $n\times n$ matrices, so they are not admissible $m\times n$ corrections when $m\neq n$. Even when $m=n$, these products are generally different matrices. 

For example, let
$$
W_0=
\begin{pmatrix}
1 & 0 & 0\\
0 & 1 & 0
\end{pmatrix},
\quad
z=
\begin{pmatrix}
1\\
2\\
-1
\end{pmatrix},
\quad
u^\star=
\begin{pmatrix}
3\\
0
\end{pmatrix}.
$$
Then $W_0z=(1,2)^\top$, $r=(2,-2)^\top$, $\|z\|_1=4$, and $\|z\|_2^2=6$. Substituting into the endpoint formulas gives
$$
\begin{aligned}
W_{\rm TACO}^\star=W_{\rm Adam}^\star
&=\frac12
\begin{pmatrix}
3 & 1 & -1\\
-1 & 1 & 1
\end{pmatrix},\\
W_{\rm Muon}^\star
&=\frac13
\begin{pmatrix}
4 & 2 & -1\\
-1 & 1 & 1
\end{pmatrix}.
\end{aligned}
$$
Direct multiplication verifies
$
W_{\rm TACO}^\star z
=W_{\rm Adam}^\star z
=W_{\rm Muon}^\star z
=u^\star
$.

Hence, both endpoints achieve zero least-squares loss. Their corrections exhibit their implicit biases:
$$
\|W_{\rm TACO}^\star-W_0\|_{\max}
=\frac12
<
\frac23
=\|W_{\rm Muon}^\star-W_0\|_{\max},
$$
whereas
$$
\|W_{\rm Muon}^\star-W_0\|_2
=\frac{2}{\sqrt{3}}
<
\frac{\sqrt{6}}{2}
=\|W_{\rm TACO}^\star-W_0\|_2.
$$

More generally, $W_0z$ is the initial output on the fixed input $z$, so $r=u^\star-W_0z$ is precisely the output correction required by fine-tuning. The dependence on $z$ records which input is being fitted; changing $z$ changes the optimization problem. For $n>1$, the target output does not determine a unique parameter matrix; the complete solution family is
$$
\{W:Wz=u^\star\}
=
W_{\rm Muon}^\star+
\{N\in\mathbb{R}^{m\times n}:Nz=0\}.
$$
Each continuation set contains exactly one matrix producing $u^\star$, so the update geometry selects one member of this solution family. 

We next turn from the limiting weights to the width dependence of the update. For $\mu$P scaling, the relevant quantity is $\|D\|_2$, which is determined by how concentrated TACO's column-wise winners are across rows. 
Assumption~\ref{assum:winner_balance} controls this concentration as \(m\) and \(n\) grow.
\setcounter{theorem}{4}
\begin{theorem}{(TACO $\mu$P-scaling)}
Let \Cref{assum:winner_balance} hold. Then the TACO direction $D=\frac{m}{n}\mathcal{S}(G)$ satisfies $\|D\|_2=\frac{m}{n}\sqrt{\max_i c_i},$ for every $G$. Hence $D$ satisfies the standard $\mu$P spectral scaling
$
\|D\|_2=\Theta\left(\sqrt{\frac{m}{n}}\right).
$
\end{theorem}
\vspace{-4mm}
\begin{proof}
Since each column of $\mathcal{S}(G)$ contains one nonzero entry at row $i_j$, we have
$$
\mathcal{S}(G)\mathcal{S}(G)^\top=\operatorname{diag}(c_1,\ldots,c_m).
$$
Hence,
$$
\|\mathcal{S}(G)\|_2=\sqrt{\max_i c_i},
$$
and the TACO direction $D=\frac{m}{n}\mathcal{S}(G)$ satisfies
$$
\|D\|_2=\frac{m}{n}\sqrt{\max_i c_i}.
$$
This identity holds for every $G$ and uses no assumptions. Dividing by $\sqrt{m/n}$ gives
$$
\frac{\|D\|_2}{\sqrt{m/n}}=\sqrt{\frac{m}{n}\max_i c_i},
$$
so the left side is $\Theta(1)$ exactly when $\max_i c_i=\Theta(n/m)$, which proves both directions.
\end{proof}
\vspace{-2mm}
The proof reduces a spectral-norm question to a counting question, since $\|\mathcal{S}(G)\|_2=\sqrt{\max_i c_i}$ exactly. The degenerate case makes the condition concrete. If a single row wins every column, then $\max_i c_i=n$ and $\|D\|_2=m/\sqrt{n}$, which exceeds the $\mu$P target by a factor of $\sqrt{m}$.

\subsection{Gradient History and Heavy Hitter Analysis}

Finally, we connect the update rule to the sparse-history design in
\Cref{sec:practical implementation}. The dense gradient history $\tilde M_t$ and
sparse gradient history $M_t$ need not agree entrywise. Rather than requiring
their selected coordinates to agree directly, we give a sufficient condition
based on the approximation error and the separation between the largest and
second-largest entries of the dense gradient history.

Let $\tilde M_t=\mu\tilde M_{t-1}+(1-\mu)G_t$ and
$M_t=\mathcal{H}_k(\mu M_{t-1}+(1-\mu)G_t)$ denote the dense and sparse gradient history. For every column $j$ and iteration $t$, define the sparse-history approximation error as
$\eps_{t,j}:=\|M_{t,:,j}-\tilde M_{t,:,j}\|_\infty$.

\begin{theorem}({Winner-margin preservation of TACO})
\label{thm:taco-exact-history}
 Let
$i^\star_{t,j}=\arg\max_i|\tilde M_{t,ij}|$ and define the dense-history
winner margin as
$
\gamma_{t,j}
=
|\tilde M_{t,i^\star_{t,j},j}|
-
\max_{i\neq i^\star_{t,j}}|\tilde M_{t,ij}|.
$
If $\gamma_{t,j}>2\eps_{t,j}$ for every column $j$ and iteration $t$,
then $\mathcal{S}(M_t)=\mathcal{S}(\tilde M_t)$ for all $t\geq0$.
\end{theorem}

\begin{proof}
Fix a column $j$ and iteration $t$, and let
$i^\star=i^\star_{t,j}$. Then,
$$
|M_{t,i^\star j}|
\geq
|\tilde M_{t,i^\star j}|-\eps_{t,j}.
$$
For any $i\neq i^\star$,
$$
|M_{t,ij}|
\leq
|\tilde M_{t,ij}|+\eps_{t,j}.
$$
Since $\gamma_{t,j}>2\eps_{t,j}$,
$$
|\tilde M_{t,i^\star j}|-\eps_{t,j}
>
\max_{i\neq i^\star}|\tilde M_{t,ij}|+\eps_{t,j}.
$$
Therefore,
$$
|M_{t,i^\star j}|>|M_{t,ij}|
$$
for every $i\neq i^\star$, so the dense and sparse gradient histories select the same winning row.

Moreover, $\gamma_{t,j}>2\eps_{t,j}$ implies
$|\tilde M_{t,i^\star j}|>\eps_{t,j}$. Hence,
$$
|M_{t,i^\star j}-\tilde M_{t,i^\star j}|
\leq\eps_{t,j}
<
|\tilde M_{t,i^\star j}|,
$$
so the selected entries also have the same sign. By the definition of
$\mathcal{S}$ in equation~(4),
$$
\mathcal{S}(M_t)_{:,j}
=
\mathcal{S}(\tilde M_t)_{:,j}.
$$
Since this holds for every column $j$, we obtain
$$
\mathcal{S}(M_t)=\mathcal{S}(\tilde M_t)
$$
for all $t\geq0$, which proves the result.
\end{proof}

The result shows that exact reconstruction of the dense gradient history is not required to preserve the TACO update. It is sufficient for the approximation error to remain smaller than half of the dense-history winner margin in every column. Thus, when the dominant coordinate is sufficiently separated from the remaining coordinates, sparse history compression can perturb the retained values while still producing exactly the same TACO direction as the dense gradient history. This provides a direct theoretical connection between TACO's winner-based update rule and the memory-efficient sparse-history design in \Cref{sec:practical implementation}.

\begin{table*}[t!]
\centering
\footnotesize

\begin{tabular}{@{}l l@{}}
\toprule
\textbf{Configuration} & \textbf{Setting} \\
\midrule

Fine-tuning
& Full-parameter \\

Model precision
& \texttt{BF16} \\

Maximum sequence length
& 1,024 \\

Gradient accumulation
& 1 \\
Learning rate & \texttt{3e-5}\\

Learning rate schedule
& Linear \\

Warmup
& 50 steps \\

Gradient checkpointing
& Enabled \\

Data seed
& 42 \\

Default optimizer step budget
& 2,000 optimizer steps \\

Default batch size
& 4 \\

TACO update
& Column-wise Top-1 \\

Gradient history EMA decay factor
& $\mu=0.95$ \\

Heavy hitter threshold
& $k=16$ entries per column \\

Gradient history value format
& \texttt{FP8 E4M3} \\

Gradient history index format
& \texttt{int32} \\

Auxiliary AdamW
& $\beta=(0.9,0.999)$, $\eps=10^{-8}$, weight decay $=0.01$ \\

\bottomrule
\end{tabular}

\caption{\small \textbf{Shared experimental configuration.}
All experiments use this setting unless otherwise specified.
}
\label{tab:shared-experiment-config}
\vspace{-4mm}
\end{table*}

\begin{table*}[t!]
\centering
\footnotesize
\setlength{\tabcolsep}{4.5pt}

\begin{tabularx}{\textwidth}{
@{}
>{\raggedright\arraybackslash}p{2.7cm}
>{\raggedright\arraybackslash}p{3.2cm}
>{\raggedright\arraybackslash}p{1.6cm}
>{\raggedright\arraybackslash}X
@{}
}
\toprule
\textbf{Experiment}
&
\textbf{Model and task}
&
\textbf{Optimizer}
&
\textbf{Experiment specific hyperparameters}
\\
\midrule


\multirow[c]{7}{3cm}{
\makecell[l]{
\textbf{E1}\\
Optimizer comparison\\
Fix model, task
}}
&
\multirow[c]{7}{4cm}{
\makecell[l]{
\textbf{Model:} OPT-13B\\
\textbf{Tasks:} SST-2, RTE,\\
BoolQ, COPA, CB,\\
MultiRC, SQuAD, \\
DROP
}}
&
TACO
&
--
\\
\cmidrule(lr){3-4}

&
&
AdamW8bit
&
LR~=~\texttt{3e-6}; 8-bit optimizer states;
\texttt{FP32} embedding; weight decay~$=0.01$
\\
\cmidrule(lr){3-4}

&
&
Adafactor
&
LR~=~\texttt{1e-4}; $\beta_2$ decay~$=-0.8$;
weight decay~$=0$
\\
\cmidrule(lr){3-4}

&
&
APOLLO-Mini
&
LR~=~\texttt{3e-6}; rank~$=1$;
projection gap~$=200$; tensor-wise scale~$=128$
\\
\cmidrule(lr){3-4}

&
&
FlashAdamW
&
LR~=~\texttt{1e-6}; \texttt{BF16} weights;
8-bit weight correction; 8-bit moments;
weight decay~$=0.01$
\\
\cmidrule(lr){3-4}

&
&
GaLore
&
LR~=~\texttt{3e-5}; rank~$=128$;
projection gap~$=200$; scale~$=0.25$;
weight decay~$=0$
\\
\cmidrule(lr){3-4}

&
&
Muon
&
LR~=~\texttt{3e-5}; momentum~$=0.95$;
Nesterov; Newton-Schulz iteration~=~$5$;
auxiliary LR~=~\texttt{3e-4}
\\

\midrule


\makecell[l]{
\textbf{E2}\\
Model scaling\\
Fix optimizer, task
}
&
\makecell[l]{
\textbf{Models:}\\
OPT-$\{$1.3B, 6.7B,\\
13B, 30B$\}$\\
Qwen3-$\{\text{8B, 14B, 32B}\}$\\
\textbf{Tasks:} \\
SST-2, RTE, BoolQ
}
&
TACO
&
--
\\

\midrule


\makecell[l]{
\textbf{E3.1}\\
Optimizer step\\
budget ablation
}
&
&
&
Steps $\in\{\text{2K, 4K, 8K, 16K}\}$
\\

\makecell[l]{
\textbf{E3.2}\\
Gradient history EMA\\
decay factor ablation
}
&
\makecell[l]{
\textbf{Model:} OPT-1.3B\\
\textbf{Tasks:}\\
BoolQ, RTE, SQuAD\\
\,
}
&
&
\makecell[l]{
$\mu\in\{0,0.95\}$; \\
additionally evaluated on \\
SST-2, COPA, CB, MultiRC, and DROP
}
\\

\makecell[l]{
\textbf{E3.3}\\
Heavy hitter \\
threshold $k$ ablation
}
&
&
&
$k\in\{\text{dense},64, 32, 16, 8, 4, 2, 1\}$
\\

\makecell[l]{
\textbf{E3.4}\\
Update sparsity\\
Top-$k$ ablation
}
&
&
TACO
&
\makecell[l]{
Top-1, Top-2, Top-4, Top-8, \\
Top-16, Top-32;\\
heavy hitter threshold $k$~=~32
}
\\

\cmidrule(lr){1-2}

\makecell[l]{
\textbf{E3.5}\\
Batch size ablation
}
&
\makecell[l]{
\textbf{Models:}\\
Qwen3-$\{\text{14B, 32B}\}$\\
\textbf{Tasks:} SST-2, RTE
}
&
&
BS $\in\{2,4,8,16,32,64,128\}$
\\

\cmidrule(lr){1-2}

\makecell[l]{
\textbf{E3.6}\\
$\mu$P learning rate\\
transfer study
}
&
\makecell[l]{
\textbf{Models:}\\
OPT-$\{$1.3B, 6.7B, \\
13B, 30B$\}$\\
Qwen3-$\{\text{8B, 14B, 32B}\}$\\
\textbf{Tasks:} SST-2, RTE
}
&
&
Learning rate sweep
\\

\midrule


\makecell[l]{
\textbf{E4}\\
Architecture transfer\\
Fix optimizer, task
}
&
\makecell[l]{
\textbf{Models:} \\
Pythia-12B, \\
Llama-3.1-8B, \\
Qwen3-32B, \\
Mistral-24B\\
\textbf{Tasks:} \\
SST-2, RTE, BoolQ
}
&
TACO
&
--
\\

\bottomrule
\end{tabularx}

\caption{\small
\textbf{Experiment configurations for E1-E4.}
E1 compares optimizers with the model and tasks fixed.
E2 studies TACO across model scales.
E3 contains the ablation studies.
E4 evaluates transfer across model architectures.
Settings not explicitly listed follow
\Cref{tab:shared-experiment-config}.
Note that TACO uses uniform hyperparameters in \Cref{tab:shared-experiment-config}, except the ablation study category, E3, where only explicitly listed hyperparameters are varied.
}
\label{tab:experiment-config}
\vspace{-4mm}
\end{table*}

\section{Experiment Setup and Additional Experiment Results}
\label{sec:experiment_setup}
\label{sec:more_experiment_results}

\smartparagraph{Model and datasets.}
For LLM experiments, we evaluate the OPT family~\citep{zhang2022opt}, including OPT-1.3B, OPT-6.7B, OPT-13B, and OPT-30B; the Qwen3 family~\citep{qwen3}, including Qwen3-8B, Qwen3-14B, and Qwen3-32B; as well as Pythia-12B~\citep{biderman2023pythia}, Llama-3.1-8B~\citep{llama3herd2024}, and Mistral-Small-24B-Base-2501~\citep{MISTRAL-CITATION}. We evaluate on SST-2~\citep{socher2013recursive_sst-2}, RTE~\citep{dagan2005pascal_rte1,bar2006second,giampiccolo2007third_rte3,bentivogli2009fifth_rte4}, BoolQ~\citep{clark-etal-2019-boolq}, CB~\citep{de2019commitmentbank}, COPA~\citep{roemmele2011choice}, and MultiRC~\citep{khashabi2018looking} from SuperGLUE~\citep{wang2019superglue}, together with the question answering datasets SQuAD~\citep{rajpurkar-etal-2016-squad} and DROP~\citep{dua-etal-2019-drop}. Our experimental code is adapted from MeZO~\citep{malladi2023finetuning}, and we adopt the same task prompts as MeZO.

\smartparagraph{Implementation details.}
TACO treats model parameters as matrices and is applied to matrix-valued parameters. 
For ordinary attention and FFN linear weight matrices, TACO uses the geometry-derived \(m/n\) scaling in \eqref{eq:vanilla TACO update}. Token embedding matrices are treated separately and use a scaling factor of 1, reflecting their one-hot lookup structure rather than an ordinary dense linear transformation.
Following the common practice of matrix-based optimizers such as Muon~\citep{jordan2024muon}, scalar and vector parameters are optimized with an auxiliary AdamW optimizer. 
For models based on the GPT-NeoX architecture, such as Pythia, the query, key, and value projections are stored in a single fused QKV parameter. We respect this logical matrix structure by treating the fused parameter as three separate Q, K, and V projections for TACO history compression and coordinate selection, and applying TACO independently to each projection before writing the updates back to the fused parameter. This prevents the Q, K, and V projections from competing for a single selected coordinate in each column.

\smartparagraph{Hyperparameters.}
We group our experiments into four categories, E1-E4; see Table~\ref{tab:experiment-config}. 
Unless otherwise specified, all experiments use the shared hyperparameters summarized in Table~\ref{tab:shared-experiment-config}.
Experiment-specific hyperparameters and settings are summarized in Table~\ref{tab:experiment-config}.

E1 compares TACO against six baseline optimizers. For each optimizer, we perform a common learning rate sweep spanning \(10^{-6}\) to \(10^{-2}\) on SST-2, RTE, and BoolQ using the same training seed, batch size 1, and a 500-step budget, and select the final learning rate based on validation performance. The selected learning rate is then used in E1 experiments.



\smartparagraph{Hardware.}
All reported experiments are conducted on NVIDIA~H100~80\,GB~HBM3~GPUs of the same model. Each device reports approximately 85.5~GB of total memory when expressed in decimal units ($10^9$ bytes). Peak GPU memory is measured as the maximum CUDA memory allocated during training and is reported in decimal GB. Runtime and throughput comparisons use the same GPU model and hardware configuration.


\subsection{Optimizer comparison}
We provide the complete OPT-13B comparison across all eight downstream tasks in Table~\ref{tab:e1-performance}. The corresponding SST-2 efficiency results, including persistent optimizer state, peak GPU memory, runtime, throughput, and validation behavior, are reported in Table~\ref{tab:opt13b_sst2_efficiency}. 
Each configuration is launched with five seeds; reported values average over successful runs. Failed runs are marked by \(n/5\) or OOM.
Although Adafactor has the smallest persistent optimizer state, its transient training memory is substantially higher, resulting in considerably larger peak GPU memory than TACO. 
FlashAdamW benefits from hardware-optimized fused kernels and gradient release, so its throughput also reflects systems-level optimization \citep{gonzalezblalock2026flashoptim}.
TACO remains the most memory-efficient method in terms of peak GPU memory while maintaining competitive downstream performance and practical training throughput.
The main performance and memory observations are discussed in \Cref{sec:experiment}.

\begin{table}[h]
\centering
\footnotesize

\begin{tabular}{@{}l c c c c c c c c@{}}
\toprule
&
\multicolumn{3}{c}{\textbf{Classification}} &
\multicolumn{3}{c}{\textbf{Multiple Choice}} &
\multicolumn{2}{c}{\textbf{Generation}} \\
\cmidrule(lr){2-4}
\cmidrule(lr){5-7}
\cmidrule(lr){8-9}

\textbf{Method}
& \textbf{SST-2}
& \textbf{RTE}
& \textbf{BoolQ}
& \textbf{COPA}
& \textbf{CB}
& \textbf{MultiRC}
& \textbf{SQuAD}
& \textbf{DROP} \\
\midrule

TACO
&94.2
&80.8
&71.5
&89.4
&87.1
&61.2
&82.6
&30.8\\


AdamW8bit
&95.3
&83.5
&62.1 {\tiny (2/5)}
&88.2
&71.8
&OOM
&88.1
&OOM\\

Adafactor
&95.5
&87.3
&82.0
&87.0
&90.0
&75.1
&88.6
&37.7\\

APOLLO-Mini
&93.0
&83.5
&77.8
&87.0
&80.7
&64.2
&88.8
&36.7\\

FlashAdamW
&95.5
&86.3
&64.8 {\tiny (4/5)}
&86.8
&85.7
&57.5
&88.4
&39.8\\

GaLore
&95.7
&85.7
&80.2
&91.8
&76.4
&67.3
&88.0
&34.5\\

Muon
&94.4
&77.0 {\tiny (4/5)}
&73.5 {\tiny (4/5)}
&87.2
&81.2 {\tiny (2/5)}
&OOM
&77.8
&OOM\\




\bottomrule
\end{tabular}
\caption{\small
\textbf{OPT-13B performance comparison across downstream tasks.} SST-2, RTE, BoolQ, COPA, and CB report accuracy; MultiRC reports F1A; SQuAD and DROP report F1. For entries marked {\scriptsize $(n/5)$}, only $n$ of the five runs complete successfully and the remaining runs fail due to OOM.
}
\label{tab:e1-performance}
\end{table}
\begin{table}[h!]
\centering
\footnotesize
\setlength{\tabcolsep}{3pt}

\begin{tabular}{l c c c c c c c}
\toprule
Optimizer
& Steps
& State Mem. $\downarrow$
& Peak Mem. $\downarrow$
& Runtime $\downarrow$
& Throughput $\uparrow$
& Acc. $\uparrow$
& Best Val. Step $\downarrow$\\

&
&
(GB)
&
(GB)
&
(minutes)
&
(tokens/s)
&
(\%)
&
\\
\midrule
TACO
& 2K & 0.159 & 27.54 & 14.92 & 275.8 & 94.22 & 1240 \\

Adafactor
& 2K & 0.013 & 52.59 & 11.77 & 349.7 & 95.50 & 840 \\

AdamW8bit
& 2K & 27.708 & 80.60 & 15.74 & 261.3 & 95.25 & 1380 \\

APOLLO-Mini
& 2K & 1.095 & 53.67 & 8.95 & 460.1 & 92.98 & 940 \\

FlashAdamW
& 2K & 40.167 & 67.68 & 7.55 & 546.7 & 95.48 & 940 \\

GaLore
& 2K & 2.655 & 56.71 & 132.34 & 31.1 & 95.71 & 1120 \\

Muon
& 2K & 26.248 & 80.01 & 101.98 & 40.3 & 94.38 & 680 \\
\bottomrule
\end{tabular}
\caption{\small
\textbf{OPT-13B fine-tuning efficiency on SST-2.} Best Val. Step reports the mean checkpoint selected by validation performance.
}
\label{tab:opt13b_sst2_efficiency}
\end{table}


\subsection{Model Scaling}
\Cref{tab:e2-scaling} provides the complete scaling results for the OPT and Qwen3 model families on SST-2, RTE, and BoolQ. These results complement the scaling discussion in \Cref{sec:experiment} and \Cref{fig:model_scaling} by reporting downstream performance and peak GPU memory for each model size. 
Results are averaged over three seeds.














\begin{table}[h!]
\centering
\footnotesize

\begin{tabular}{@{}l c cc cc cc@{}}
\toprule
\textbf{Family}
& \textbf{Model Size}
& \multicolumn{2}{c}{\textbf{SST-2}}
& \multicolumn{2}{c}{\textbf{RTE}}
& \multicolumn{2}{c}{\textbf{BoolQ}} \\
\cmidrule(lr){3-4}
\cmidrule(lr){5-6}
\cmidrule(lr){7-8}

&
& \textbf{Acc.}
& \textbf{Mem. (GB)}
& \textbf{Acc.}
& \textbf{Mem. (GB)}
& \textbf{Acc.}
& \textbf{Mem. (GB)} \\
\midrule

OPT
& 1.3B
& 93.9 & 3.4
& 72.6 & 3.9
& 68.6 & 5.7 \\

& 6.7B
& 94.8 & 14.8
& 75.6 & 15.6
& 70.9 & 17.5 \\

& 13B
& 94.2 & 27.5
& 81.2 & 28.8
& 72.7 & 31.3 \\

& 30B
& 94.6 & 62.5
& 68.4 & 64.7
& 63.4 & 68.1 \\

\midrule

Qwen3
& 8B
& 93.4 & 19.4
& 89.5 & 20.7
& 85.5 & 22.9 \\

& 14B
& 94.5 & 33.3
& 90.5 & 34.8
& 87.9 & 36.8 \\

& 32B
& 92.1 & 69.5
& 89.8 & 71.6
& 88.1 & 75.0 \\

\bottomrule
\end{tabular}

\caption{\small
Scaling behavior of TACO across the OPT and Qwen3 model families.
Accuracy and peak GPU memory are averaged over completed seeds.
}
\label{tab:e2-scaling}
\end{table}

\subsection{Ablation studies}

\smartparagraph{E3.1: Optimizer step budget ablation.}
We vary the optimizer step budget from 2K to 16K on OPT-1.3B; results are reported in \Cref{tab:e3-steps}. TACO continues to improve beyond the default 2K-step budget across all three tasks. Most of the improvement on RTE and SQuAD is obtained by 4K-8K steps, while BoolQ continues to improve through 16K steps. Relative to the 2K-step setting, the 16K-step budget improves BoolQ and RTE by 6.5 and 4.3 percentage points, respectively, and SQuAD by 2.1 F1 points.
These results indicate that the default 2K-step budget does not represent TACO's performance ceiling. 
Results are averaged over three seeds.




\begin{table}[h!]
\centering
\footnotesize

\begin{tabular}{@{}l c c c c c@{}}
\toprule
\textbf{Task}
& $\boldsymbol{T=2\mathrm{K}}$
& $\boldsymbol{T=4\mathrm{K}}$
& $\boldsymbol{T=8\mathrm{K}}$
& $\boldsymbol{T=16\mathrm{K}}$
& \textbf{Change} \\
\midrule
BoolQ  & 68.6 & 72.6 & 73.3 & \textbf{75.1} & $+6.5$ pp \\
RTE    & 72.6 & 76.5 & 76.7 & \textbf{76.9} & $+4.3$ pp \\
SQuAD  & 78.6 & 80.0 & \textbf{80.8} & 80.7 & $+2.1$ F1 \\
\bottomrule
\end{tabular}

\caption{\small TACO optimizer step budget ablation on OPT-1.3B with $T\in\{2\mathrm{K},4\mathrm{K},8\mathrm{K},16\mathrm{K}\}$. The change is measured from $T=2\mathrm{K}$ to $T=16\mathrm{K}$.}
\label{tab:e3-steps}
\end{table}




\smartparagraph{E3.2: Gradient history EMA decay factor ablation.}
We examine the effect of the gradient history EMA decay factor by comparing the default $\mu=0.95$ with the history-free setting $\mu=0$ across all eight downstream tasks; results are reported in \Cref{tab:e3-history}. The EMA gradient history improves performance on all six classification and multiple-choice tasks, with the largest gains on MultiRC ($+8.9$ F1A) and CB ($+7.1$ percentage points). The two generation tasks, SQuAD and DROP, favor $\mu=0$ by 2.6 and 0.4 F1 points, respectively. Overall, these results support $\mu=0.95$ as a uniform practical default. Results are averaged over three seeds.















\begin{table}[h!]
\centering
\footnotesize

\begin{tabular}{@{}l c c c@{}}
\toprule
\textbf{Task}
& $\boldsymbol{\mu=0}$
& $\boldsymbol{\mu=0.95}$
& \textbf{Change} \\
\midrule
SST-2   & 93.7 & 93.9 & $+0.3$ pp \\
RTE     & 71.2 & 72.6 & $+1.3$ pp \\
BoolQ   & 66.1 & 68.6 & $+2.5$ pp \\
COPA    & 78.0 & 79.0 & $+1.0$ pp \\
CB      & 79.2 & 86.3 & $+7.1$ pp \\
MultiRC & 53.5 & 62.5 & $+8.9$ F1A \\
SQuAD   & 81.3 & 78.6 & $-2.6$ F1 \\
DROP    & 28.7 & 28.3 & $-0.4$ F1 \\
\bottomrule
\end{tabular}

\caption{
Effect of gradient history EMA decay factor $\mu$ on OPT-1.3B.
SST-2, RTE, BoolQ, COPA, and CB report accuracy; MultiRC reports F1A; SQuAD and DROP report F1.
}
\label{tab:e3-history}
\end{table}

\smartparagraph{E3.3: Heavy hitter threshold $k$ ablation.}
We compare dense gradient histories with sparse gradient history retaining $k\in\{64,32,16,8,4,2,1\}$ entries per column; results are reported in \Cref{tab:e3-history-k}. Moderately sparse gradient histories retain performance broadly comparable to dense gradient history while substantially reducing persistent optimizer state. In particular, $k=16$ achieves the strongest BoolQ and RTE performance while using only 38.3~MB of optimizer state, a $34.4\times$ reduction from the 1317.7~MB required by dense gradient history. SQuAD is slightly better with dense or $k=32$ history, while more aggressive compression at very small $k$ begins to degrade performance. We therefore use $k=16$ as the default heavy hitter threshold. Results are averaged over three seeds.






\begin{table}[h!]
\centering
\footnotesize

\begin{tabular}{@{}l c c c c c c c c@{}}
\toprule
& \textbf{Dense}
& $\boldsymbol{k=64}$
& $\boldsymbol{k=32}$
& $\boldsymbol{k=16}$
& $\boldsymbol{k=8}$
& $\boldsymbol{k=4}$
& $\boldsymbol{k=2}$
& $\boldsymbol{k=1}$ \\
\midrule

BoolQ
& 68.1
& 66.8
& 67.9
& \textbf{68.6}
& 67.5
& 65.5
& 65.8
& 66.3 \\

RTE
& 69.7
& 70.9
& 71.0
& \textbf{72.6}
& 70.8
& 70.3
& 69.3
& 69.7 \\

SQuAD
& 79.4
& 79.4
& 79.4
& 78.6
& \textbf{78.9}
& 78.6
& 77.7
& 78.4 \\

\midrule
State memory (MB)
& 1317.7
& 145.4
& 74.0
& 38.3
& 20.4
& 11.5
& 7.0
& 4.8 \\

\bottomrule
\end{tabular}%

\caption{\small TACO heavy hitter threshold $k$ ablation on OPT-1.3B. }
\label{tab:e3-history-k}
\end{table}

\smartparagraph{E3.4: Update sparsity Top-$k$ ablation.}
We relax TACO's theoretically derived column-wise Top-1 update and evaluate Top-$k$ variants while fixing the heavy hitter threshold to $k=32$; results are reported in \Cref{tab:e3-topk}. Top-1 achieves the strongest final performance across all three tasks. Top-2 remains broadly comparable, while performance begins to degrade for denser updates and deteriorates substantially at Top-8 and beyond, particularly on BoolQ and RTE. Although Top-2 and Top-4 can make faster progress during the early stage of training on some tasks, this advantage is transient and does not translate into better final performance. These results support retaining the column-wise Top-1 update used by TACO, which also corresponds exactly to the steepest-descent solution under the $1\to1$ geometry.
Results are averaged over three seeds.



\begin{table}[h!]
\centering
\footnotesize

\begin{tabular}{@{}l c c c c c c@{}}
\toprule
\textbf{Task}
& \textbf{Top-1 (TACO)}
& \textbf{Top-2}
& \textbf{Top-4}
& \textbf{Top-8}
& \textbf{Top-16}
& \textbf{Top-32} \\
\midrule

BoolQ
& \textbf{67.9}
& 66.6
& 67.8
& 61.9
& 61.5
& 61.5 \\

RTE
& \textbf{71.0}
& 70.4
& 69.3
& 53.6
& 56.1
& 47.1 \\

SQuAD
& \textbf{79.4}
& 79.0
& 76.8
& 76.0
& 73.0
& 70.7 \\

\bottomrule
\end{tabular}

\caption{\small 
TACO update Top-$k$ ablation on OPT-1.3B with heavy hitter threshold $k=32$.
}
\label{tab:e3-topk}
\end{table}





\smartparagraph{E3.5: TACO and batch sizes.}
We evaluate batch sizes from 2 to 128 on Qwen3-14B and Qwen3-32B using SST-2 and RTE; results are reported in \Cref{tab:taco_batch_size}. Larger batches increase peak GPU memory without consistently improving accuracy, so we retain batch size 4 as the default.

\begin{table*}[h!]
\centering
\footnotesize

\setlength{\tabcolsep}{4pt}

\resizebox{\textwidth}{!}{%
\begin{tabular}{@{}r cc cc cc cc@{}}
\toprule
& \multicolumn{4}{c}{\textbf{Qwen3-14B}}
& \multicolumn{4}{c}{\textbf{Qwen3-32B}} \\
\cmidrule(lr){2-5}\cmidrule(lr){6-9}

& \multicolumn{2}{c}{\textbf{SST-2}}
& \multicolumn{2}{c}{\textbf{RTE}}
& \multicolumn{2}{c}{\textbf{SST-2}}
& \multicolumn{2}{c}{\textbf{RTE}} \\
\cmidrule(lr){2-3}\cmidrule(lr){4-5}
\cmidrule(lr){6-7}\cmidrule(lr){8-9}

\textbf{Batch}
& \textbf{Acc.} & \textbf{Mem. (GB)}
& \textbf{Acc.} & \textbf{Mem. (GB)}
& \textbf{Acc.} & \textbf{Mem. (GB)}
& \textbf{Acc.} & \textbf{Mem. (GB)} \\
\midrule

2
& $94.1 \pm 0.9$ & 33.1
& $89.5 \pm 1.2$ & 33.8
& $92.3 \pm 2.6$ & 69.3
& $89.4 \pm 1.9$ & 70.3 \\

4
& $94.2 \pm 1.0$ & 33.3
& $89.2 \pm 1.3$ & 34.8
& $93.3 \pm 1.6$ & 69.6
& $\mathbf{90.1} \pm 0.8$ & 71.5 \\

8
& $93.9 \pm 0.4$ & 33.6
& $89.3 \pm 1.9$ & 36.6
& $93.7 \pm 0.5$ & 70.0
& $89.4 \pm 0.9$ & 74.0 \\

16
& $94.5 \pm 0.6$ & 34.3
& $\mathbf{89.8} \pm 0.8$ & 43.1
& $93.1 \pm 0.9$ & 71.0
& $89.3 \pm 0.6$ & 81.4 \\

32
& $93.9 \pm 1.7$ & 35.8
& $89.0 \pm 2.1$ & 56.4
& $\mathbf{94.0} \pm 0.7$ & 72.9
& \multicolumn{2}{c}{OOM} \\

64
& $\mathbf{94.8} \pm 0.9$ & 40.1
& \multicolumn{2}{c}{OOM}
& $93.5 \pm 0.8$ & 78.1
& \multicolumn{2}{c}{OOM} \\

128
& $94.1 \pm 1.0$ & 50.3
& \multicolumn{2}{c}{OOM}
& \multicolumn{2}{c}{OOM}
& \multicolumn{2}{c}{OOM} \\

\bottomrule
\end{tabular}%
}
\caption{\small TACO batch size ablation on Qwen3-14B and Qwen3-32B.
Accuracy is reported as mean $\pm$ sample standard deviation, and memory as peak GPU memory.}
\label{tab:taco_batch_size}
\end{table*}


\smartparagraph{E3.6: $\mu$P Learning Rate Transfer in TACO.} 

We sweep learning rates across Qwen3-8B/14B/32B on SST-2 and OPT-1.3B/6.7B/13B/30B on SST-2 and RTE, using 1,000 optimizer steps, batch size 4, and five training seeds. As shown in \Cref{fig:muP}, every model-dataset configuration selects \(3\times10^{-5}\) as the best learning rate, providing empirical support for learning rate transfer across model scales. This contrasts with conventional scale-dependent tuning; for example, \citet{brown2020language} use progressively smaller learning rates as GPT-3 model size increases.


\begin{figure}[h!]
    \centering
    \includegraphics[width=0.99\linewidth]{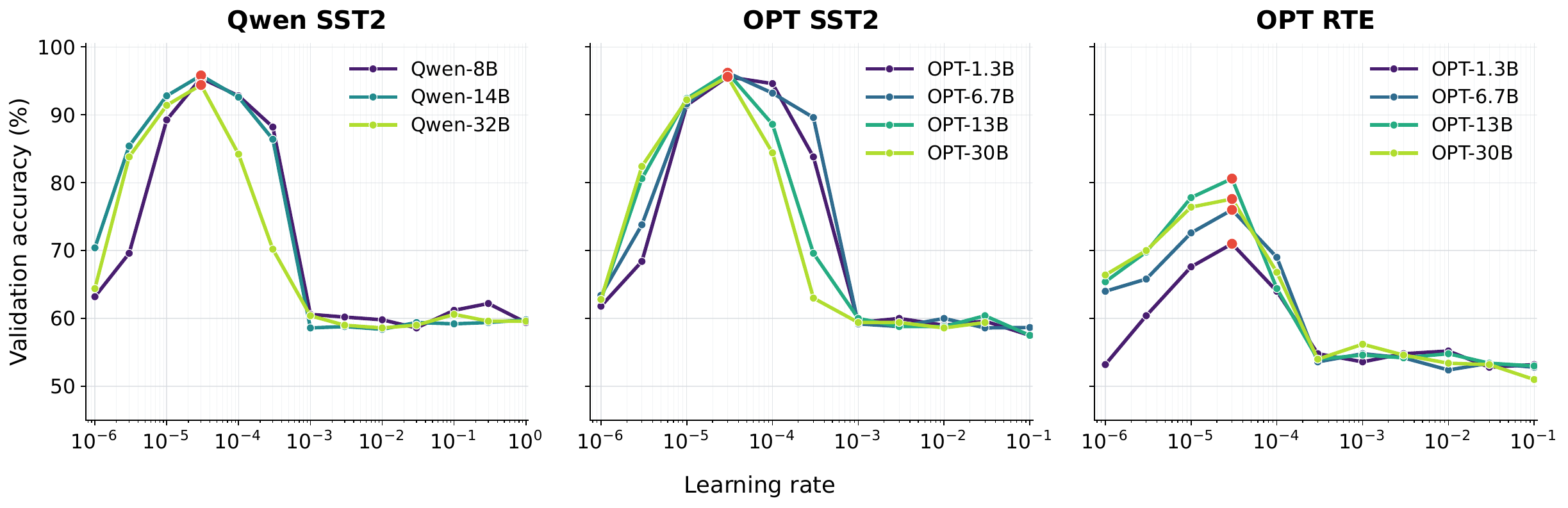}
    \caption{\small Every evaluated model and dataset configuration attains its best validation performance at \(3\times10^{-5}\), empirically supporting the \(\mu\)P scaling result in \Cref{thm:taco_mup}}
    \label{fig:muP}
\end{figure}

\subsection{Architecture Transfer}

The complete architecture transfer results for Pythia-12B, Llama-3.1-8B, Qwen3-32B, and Mistral-24B are reported and discussed in \Cref{sec:experiment}; see \Cref{fig:architecture_table_opt_rte}.

{

\section{Extended Experiments}
\label{sec:extended_experiments}

This section provides extended empirical evaluations that complement \Cref{sec:experiment} and \Cref{sec:more_experiment_results}, covering: (i) controlled backward-gradient-release baselines, TACO-FP16, and matched wallclock optimization budgets (\Cref{subsec:controlled-hook-exp}); (ii) history precision and effective \texttt{BF16} parameter update diagnostics (\Cref{subsec:precision-diagnostics}); (iii) empirical row winner concentration and normalized spectral ratios across model widths (\Cref{subsec:winner-diagnostics}); and (iv) a 4-GPU distributed training as a proof of concept of TACO's communication efficiency in this setup (\Cref{subsec:distributed-exp}).

\subsection{Controlled Gradient Release Baselines, TACO-FP16, and Matched Time Evaluation}
\label{subsec:controlled-hook-exp}

To disentangle TACO's sparse optimizer state from backward gradient release and quantify the impact of history precision, we reran OPT-13B under a unified 2,000 steps, batch size 4 configuration across three training seeds, evaluating both TACO (\texttt{E4M3} heavy hitter history) and TACO-FP16 (\texttt{FP16} heavy hitter history) alongside AdaLOMO~\citep{lv2023adalomo}, GaLore~\citep{galore}, Adafactor~\citep{adafactor}, Adafactor with autograd hooks, AdamW8bit~\citep{adamw8bit}, and AdamW8bit with autograd hooks. The hooked wrappers apply parameter updates and release gradients inside post accumulation \texttt{autograd hooks} during backpropagation while preserving each optimizer's numerical update rule. For each optimizer in this controlled comparison, learning rates were selected via a two stage search; details and selected learning rates are reported in \Cref{tab:e1-lr-sweep}.

\begin{table*}[h!]
\centering
\footnotesize
\setlength{\tabcolsep}{4pt}
\resizebox{\textwidth}{!}{%
\begin{tabular}{@{}l cccccccc@{}}
\toprule
\textbf{Optimizer} & \textbf{SST-2} & \textbf{RTE} & \textbf{BoolQ} & \textbf{COPA} & \textbf{CB} & \textbf{MultiRC} & \textbf{SQuAD} & \textbf{DROP} \\
\midrule
TACO (\texttt{E4M3}) & $94.5 \pm 0.3$ & $81.1 \pm 1.8$ & $74.5 \pm 0.9$ & $88.0 \pm 1.0$ & $80.4 \pm 3.1$ & $57.3 \pm 1.0$ & $83.8 \pm 1.5$ & $32.1 \pm 0.5$ \\
TACO-FP16 & $95.6 \pm 0.6$ & $84.6 \pm 0.6$ & $80.9 \pm 0.5$ & $87.7 \pm 2.1$ & $88.7 \pm 1.0$ & $68.7 \pm 1.1$ & $88.5 \pm 0.5$ & $38.9 \pm 0.5$ \\
AdaLOMO & $96.3 \pm 0.3$ & $87.2 \pm 0.6$ & $85.3 \pm 0.7$ & $89.7 \pm 2.5$ & $94.1 \pm 1.0$ & $75.6 \pm 1.3$ & $89.4 \pm 0.3$ & $42.6 \pm 2.6$ \\
GaLore & $95.3 \pm 0.1$ & $84.4 \pm 4.1$ & $81.6 \pm 0.5$ & $89.3 \pm 3.2$ & $89.9 \pm 4.5$ & $65.6 \pm 3.1$ & $87.9 \pm 0.3$ & $34.9 \pm 0.9$ \\
Adafactor & $95.5 \pm 0.6$ & $84.8 \pm 1.1$ & $84.4 \pm 0.4$ & $91.0 \pm 1.0$ & $91.1 \pm 0.0$ & $77.6 \pm 1.6$ & $89.4 \pm 0.5$ & $37.4 \pm 2.8$ \\
Adafactor + hooks & $95.5 \pm 0.6$ & $84.8 \pm 1.1$ & $84.4 \pm 0.4$ & $88.0 \pm 1.7$ & $92.9 \pm 1.8$ & $77.6 \pm 1.6$ & $89.4 \pm 0.5$ & $37.0 \pm 2.1$ \\
AdamW8bit & $96.1 \pm 0.3$ & $84.0 \pm 0.4$ & OOM & $92.7 \pm 1.5$ & $76.2 \pm 16.6$ & OOM & $88.2 \pm 0.4$ & OOM \\
AdamW8bit + hooks & $96.1 \pm 0.3$ & $84.0 \pm 0.4$ & $63.4 \pm 0.8$ & $87.0 \pm 7.6$ & $73.8 \pm 12.2$ & $56.9 \pm 0.4$ & $88.2 \pm 0.4$ & $34.4 \pm 0.7$ \\
\bottomrule
\end{tabular}%
}
\caption{\small \textbf{OPT-13B performance comparison across downstream tasks.} Performance is reported as mean $\pm$ std over 3 seeds. }
\label{tab:e1-rebuttal-acc}
\end{table*}

TACO-FP16 improves over unscaled \texttt{E4M3} TACO on 7 of 8 tasks while adding only 29.7 MB of persistent state. 

\begin{table*}[h!]
\centering
\footnotesize
\setlength{\tabcolsep}{4pt}
\resizebox{\textwidth}{!}{%
\begin{tabular}{@{}l c c c c c@{}}
\toprule
\textbf{Optimizer} & \textbf{State (GB)} & \textbf{Peak (GB)} & \textbf{Runtime (min)} & \textbf{Throughput (tokens/s)} & \textbf{F1 score} \\
\midrule
TACO (\texttt{E4M3}) & 0.1590 & $31.0 \pm 0.1$ & $24.1 \pm 0.1$ & $1776.5 \pm 3.0$ & $32.1 \pm 0.5$ \\
TACO-FP16 & 0.1886 & $31.0 \pm 0.1$ & $24.2 \pm 0.2$ & $1769.8 \pm 17.2$ & $38.9 \pm 0.5$ \\
AdaLOMO & 0.0257 & $30.9 \pm 0.1$ & $42.4 \pm 0.6$ & $1009.2 \pm 15.3$ & $42.6 \pm 2.6$ \\
GaLore & 2.6551 & $33.9 \pm 0.3$ & $132.5 \pm 0.5$ & $322.8 \pm 1.3$ & $34.9 \pm 0.9$ \\
Adafactor & 0.0128 & $54.2 \pm 0.0$ & $20.5 \pm 0.1$ & $2081.5 \pm 9.8$ & $37.4 \pm 2.8$ \\
Adafactor + autograd hooks & 0.0128 & $30.8 \pm 0.1$ & $22.2 \pm 0.3$ & $1929.0 \pm 21.1$ & $37.0 \pm 2.1$ \\
AdamW8bit & - & OOM & OOM & OOM & OOM \\
AdamW8bit + autograd hooks & 27.7076 & $58.6 \pm 0.1$ & $26.5 \pm 1.1$ & $1614.1 \pm 65.5$ & $34.4 \pm 0.7$ \\
\bottomrule
\end{tabular}
}
\caption{\small \textbf{OPT-13B fine-tuning efficiency on DROP.} Data are reported as mean $\pm$ std over 3 seeds.}
\label{tab:e1-rebuttal-eff}
\end{table*}

Backward gradient release lowers Adafactor peak memory by $23.39$\,GB to the $\sim31$\,GB VRAM floor, where TACO-FP16 achieves $1.75\times$ higher throughput than AdaLOMO, $5.5\times$ higher throughput than GaLore, higher DROP F1 than Adafactor + autograd hooks, and $\gO(n)$ per-matrix parameter updates. 

\begin{table*}[h!]
\centering
\footnotesize
\setlength{\tabcolsep}{3.5pt}
\resizebox{\textwidth}{!}{%
\begin{tabular}{@{}ll ccccccc@{}}
\toprule
\textbf{Dataset} & \textbf{Optimizer} & \textbf{0 min} & \textbf{5 min} & \textbf{10 min} & \textbf{20 min} & \textbf{30 min} & \textbf{45 min} & \textbf{60 min} \\
\midrule
\multirow{3}{*}{\textbf{SST-2}}
 & TACO-FP16 & $61.0 \pm 0.0$ & $94.8 \pm 0.6$ & $94.3 \pm 1.0$ & $93.7 \pm 0.4$ & $93.7 \pm 0.4$ & $94.8 \pm 1.8$ & $94.8 \pm 1.8$ \\
 & AdaLOMO & $61.0 \pm 0.0$ & $95.0 \pm 0.2$ & $95.2 \pm 0.3$ & $95.1 \pm 0.6$ & $95.8 \pm 0.1$ & $95.5 \pm 0.7$ & $95.3 \pm 0.6$ \\
 & Adafactor + hooks & $61.0 \pm 0.0$ & $95.3 \pm 0.2$ & $95.5 \pm 0.2$ & $95.6 \pm 0.1$ & $95.8 \pm 0.4$ & $95.8 \pm 0.4$ & $95.8 \pm 0.4$ \\
\midrule
\multirow{3}{*}{\textbf{RTE}}
 & TACO-FP16 & $57.4 \pm 0.0$ & $79.5 \pm 2.3$ & $79.5 \pm 2.3$ & $82.2 \pm 2.0$ & $83.0 \pm 1.9$ & $82.7 \pm 1.6$ & $82.7 \pm 1.6$ \\
 & AdaLOMO & $57.4 \pm 0.0$ & $80.0 \pm 2.4$ & $82.0 \pm 1.3$ & $86.2 \pm 0.9$ & $86.5 \pm 1.5$ & $86.5 \pm 1.5$ & $86.5 \pm 1.5$ \\
 & Adafactor + hooks & $57.4 \pm 0.0$ & $85.6 \pm 2.2$ & $86.8 \pm 1.1$ & $86.8 \pm 1.1$ & $85.9 \pm 0.6$ & $85.6 \pm 0.0$ & $85.3 \pm 0.8$ \\
\midrule
\multirow{3}{*}{\textbf{BoolQ}}
 & TACO-FP16 & $59.3 \pm 0.0$ & $75.9 \pm 1.3$ & $79.3 \pm 0.5$ & $79.8 \pm 1.2$ & $81.0 \pm 1.5$ & $80.9 \pm 1.4$ & $80.7 \pm 1.0$ \\
 & AdaLOMO & $59.3 \pm 0.0$ & $75.2 \pm 3.0$ & $77.9 \pm 0.3$ & $80.1 \pm 2.7$ & $81.7 \pm 1.6$ & $83.4 \pm 1.0$ & $83.8 \pm 2.8$ \\
 & Adafactor + hooks & $59.3 \pm 0.0$ & $76.6 \pm 3.8$ & $80.8 \pm 3.5$ & $80.9 \pm 3.6$ & $81.9 \pm 4.7$ & $82.5 \pm 5.0$ & $82.2 \pm 4.7$ \\
\midrule
\multirow{3}{*}{\textbf{COPA}}
 & TACO-FP16 & $79.0 \pm 0.0$ & $86.0 \pm 1.7$ & $87.7 \pm 2.5$ & $88.0 \pm 2.7$ & $88.0 \pm 2.7$ & $90.0 \pm 1.0$ & $90.0 \pm 1.0$ \\
 & AdaLOMO & $79.0 \pm 0.0$ & $87.3 \pm 0.6$ & $88.3 \pm 1.2$ & $88.3 \pm 0.6$ & $89.0 \pm 1.0$ & $89.0 \pm 1.0$ & $89.0 \pm 1.0$ \\
 & Adafactor + hooks & $79.0 \pm 0.0$ & $89.3 \pm 1.5$ & $88.7 \pm 0.6$ & $89.0 \pm 0.0$ & $89.0 \pm 0.0$ & $89.0 \pm 0.0$ & $89.0 \pm 0.0$ \\
\midrule
\multirow{3}{*}{\textbf{CB}}
 & TACO-FP16 & $44.6 \pm 0.0$ & $86.9 \pm 1.0$ & $87.5 \pm 1.8$ & $86.9 \pm 6.3$ & $89.3 \pm 3.1$ & $89.3 \pm 3.1$ & $89.3 \pm 3.1$ \\
 & AdaLOMO & $44.6 \pm 0.0$ & $91.7 \pm 9.0$ & $93.5 \pm 6.3$ & $93.5 \pm 6.3$ & $93.5 \pm 6.3$ & $93.5 \pm 6.3$ & $93.5 \pm 6.3$ \\
 & Adafactor + hooks & $44.6 \pm 0.0$ & $94.1 \pm 4.1$ & $94.1 \pm 4.1$ & $94.1 \pm 4.1$ & $94.6 \pm 3.1$ & $94.6 \pm 3.1$ & $94.6 \pm 3.1$ \\
\midrule
\multirow{3}{*}{\textbf{MultiRC}}
 & TACO-FP16 & $56.7 \pm 0.0$ & $56.7 \pm 0.0$ & $60.6 \pm 3.5$ & $64.8 \pm 3.8$ & $68.2 \pm 1.7$ & $70.2 \pm 1.7$ & $72.6 \pm 0.5$ \\
 & AdaLOMO & $56.7 \pm 0.0$ & $56.7 \pm 0.0$ & $58.8 \pm 4.7$ & $71.0 \pm 0.8$ & $73.3 \pm 2.6$ & $73.3 \pm 2.6$ & $76.2 \pm 1.8$ \\
 & Adafactor + hooks & $56.7 \pm 0.0$ & $62.6 \pm 9.9$ & $62.6 \pm 9.9$ & $67.6 \pm 1.4$ & $74.2 \pm 1.4$ & $74.8 \pm 1.7$ & $75.6 \pm 2.9$ \\
\midrule
\multirow{3}{*}{\textbf{SQuAD}}
 & TACO-FP16 & $44.9 \pm 0.0$ & $87.0 \pm 0.2$ & $87.3 \pm 0.4$ & $87.7 \pm 0.4$ & $87.8 \pm 0.4$ & $88.4 \pm 0.7$ & $88.7 \pm 0.7$ \\
 & AdaLOMO & $44.9 \pm 0.0$ & $86.4 \pm 0.5$ & $87.0 \pm 1.0$ & $87.9 \pm 0.5$ & $88.4 \pm 0.5$ & $88.7 \pm 0.8$ & $88.7 \pm 0.8$ \\
 & Adafactor + hooks & $44.9 \pm 0.0$ & $87.7 \pm 0.3$ & $86.9 \pm 1.0$ & $89.2 \pm 0.8$ & $90.1 \pm 0.5$ & $90.2 \pm 1.0$ & $90.0 \pm 0.2$ \\
\midrule
\multirow{3}{*}{\textbf{DROP}}
 & TACO-FP16 & $28.3 \pm 0.0$ & $32.7 \pm 0.3$ & $34.9 \pm 0.8$ & $37.0 \pm 0.2$ & $38.6 \pm 1.2$ & $39.5 \pm 0.5$ & $40.7 \pm 0.9$ \\
 & AdaLOMO & $28.3 \pm 0.0$ & $34.5 \pm 1.2$ & $35.4 \pm 1.1$ & $38.1 \pm 1.2$ & $38.6 \pm 0.2$ & $41.3 \pm 2.2$ & $41.7 \pm 2.8$ \\
 & Adafactor + hooks & $28.3 \pm 0.0$ & $34.7 \pm 0.3$ & $36.0 \pm 1.1$ & $36.0 \pm 1.1$ & $37.8 \pm 4.0$ & $40.1 \pm 3.3$ & $43.0 \pm 1.8$ \\
\bottomrule
\end{tabular}%
}
\caption{\small \textbf{Matched wall-clock optimization time comparison on OPT-13B across all eight tasks.} Performance are reported as mean $\pm$ std over 3 seeds.}
\label{tab:e1-matched-time}
\end{table*}

\subsubsection{Learning Rate selection}

In this Section, we present the 2,000 steps validation selected learning rate; see \Cref{tab:e1-lr-sweep}.

\begin{table*}[h!]
\centering
\footnotesize
\setlength{\tabcolsep}{5pt}
\begin{tabular}{@{}l cccccccc@{}}
\toprule
\textbf{Optimizer} & \textbf{SST-2} & \textbf{RTE} & \textbf{BoolQ} & \textbf{COPA} & \textbf{CB} & \textbf{MultiRC} & \textbf{SQuAD} & \textbf{DROP} \\
\midrule
TACO (\texttt{E4M3}) & \texttt{3e-5} & \texttt{3e-5} & \texttt{1e-5} & \texttt{3e-5} & \texttt{3e-5} & \texttt{1e-5} & \texttt{3e-5} & \texttt{1e-5} \\
TACO-FP16 & \texttt{3e-5} & \texttt{3e-5} & \texttt{1e-5} & \texttt{3e-5} & \texttt{3e-5} & \texttt{1e-5} & \texttt{3e-5} & \texttt{3e-5} \\
AdaLOMO & \texttt{3e-4} & \texttt{3e-4} & \texttt{3e-4} & \texttt{1e-3} & \texttt{1e-3} & \texttt{3e-4} & \texttt{3e-4} & \texttt{3e-4} \\
GaLore & \texttt{3e-5} & \texttt{1e-4} & \texttt{3e-5} & \texttt{1e-4} & \texttt{1e-4} & \texttt{3e-5} & \texttt{1e-4} & \texttt{3e-5} \\
Adafactor & \texttt{3e-5} & \texttt{1e-4} & \texttt{1e-4} & \texttt{1e-4} & \texttt{1e-4} & \texttt{1e-4} & \texttt{3e-4} & \texttt{1e-4} \\
Adafactor + hooks& \texttt{3e-5} & \texttt{1e-4} & \texttt{1e-4} & \texttt{1e-4} & \texttt{3e-4} & \texttt{1e-4} & \texttt{1e-4} & \texttt{1e-4} \\
AdamW8bit & \texttt{3e-6} & \texttt{3e-6} & OOM & \texttt{1e-5} & \texttt{1e-5} & OOM & \texttt{3e-6} & OOM \\
AdamW8bit + hooks & \texttt{3e-6} & \texttt{3e-6} & \texttt{3e-7} & \texttt{1e-5} & \texttt{1e-5} & \texttt{3e-5} & \texttt{3e-6} & \texttt{1e-5} \\
\bottomrule
\end{tabular}
\caption{\small \textbf{Validation selected learning rates for OPT-13B under the 2000 steps, batch size 4 configuration.}}
\label{tab:e1-lr-sweep}
\end{table*}

\subsection{History Precision and Effective \texttt{BF16} Parameter Updates}
\label{subsec:precision-diagnostics}

Because weights are stored in \texttt{BF16} without an \texttt{FP32} master copy, small updates at $\mathrm{lr}=3\times 10^{-5}$ can round to zero when $|w| > 2^{-7} \approx 0.0078$, and unscaled \texttt{FP8 E4M3} history entries below $2^{-10} \approx 0.00098$ underflow to zero. \Cref{tab:precision-global} reports the fraction of columnwise Top-1 selections with nonzero stored history and the fraction of selected updates that actually modify the stored \texttt{BF16} parameter on OPT-13B, while \Cref{tab:precision-layerwise} breaks down these diagnostics across representative transformer blocks (Layers 0, 12, and 23) and projection types. Switching from unscaled \texttt{E4M3} to \texttt{TACO-FP16} ($+29.7$\,MB state) eliminates history underflow and substantially increases effective \texttt{BF16} updates across all layers.

\begin{table}[h!]
\centering
\footnotesize
\setlength{\tabcolsep}{6pt}
\begin{tabular}{@{}l ccc@{}}
\toprule
\textbf{Metric} & \textbf{SST-2} & \textbf{BoolQ} & \textbf{MultiRC} \\
\midrule
E4M3 Nonzero Top-1 (\%) & 55.1 & 65.0 & 47.8 \\
E4M3 Effective Update (\%) & 22.0 & 23.2 & 19.3 \\
FP16 Nonzero Top-1 (\%) & 93.3 & 94.9 & 91.6 \\
FP16 Effective Update (\%) & 34.8 & 37.8 & 40.9 \\
\bottomrule
\end{tabular}
\caption{\small \textbf{Precision and update rates on OPT-13B.}}
\label{tab:precision-global}
\end{table}

\begin{table*}[h!]
\centering
\scriptsize
\setlength{\tabcolsep}{3.5pt}
\resizebox{\textwidth}{!}{%
\begin{tabular}{@{}cl cccccc@{}}
\toprule
\textbf{Layer} & \textbf{Projection} & \textbf{SST-2 / E4M3} & \textbf{SST-2 / FP16} & \textbf{BoolQ / E4M3} & \textbf{BoolQ / FP16} & \textbf{MultiRC / E4M3} & \textbf{MultiRC / FP16} \\
\midrule
0 & q\_proj & 90.4 / 32.3 & 100.0 / 42.3 & 98.5 / 21.8 & 100.0 / 37.1 & 89.3 / 21.8 & 100.0 / 37.0 \\
0 & k\_proj & 84.5 / 28.7 & 100.0 / 41.7 & 99.3 / 23.7 & 100.0 / 38.6 & 99.5 / 26.5 & 100.0 / 38.8 \\
0 & v\_proj & 100.0 / 74.5 & 100.0 / 78.8 & 100.0 / 62.3 & 100.0 / 79.4 & 100.0 / 70.5 & 100.0 / 80.2 \\
0 & out\_proj & 97.3 / 84.9 & 100.0 / 94.6 & 99.8 / 71.3 & 100.0 / 94.8 & 96.0 / 54.9 & 100.0 / 94.9 \\
0 & fc1 & 99.6 / 84.4 & 100.0 / 90.4 & 100.0 / 73.1 & 100.0 / 94.9 & 100.0 / 78.8 & 100.0 / 94.6 \\
0 & fc2 & 47.9 / 3.6 & 98.2 / 9.9 & 34.5 / 2.7 & 98.7 / 10.3 & 10.0 / 0.7 & 98.7 / 10.3 \\
\midrule
12 & q\_proj & 75.9 / 19.2 & 100.0 / 30.3 & 53.5 / 17.5 & 100.0 / 34.3 & 52.6 / 20.3 & 100.0 / 55.8 \\
12 & k\_proj & 69.1 / 17.1 & 100.0 / 37.2 & 91.1 / 26.5 & 100.0 / 40.0 & 89.4 / 23.4 & 100.0 / 50.6 \\
12 & v\_proj & 98.0 / 38.1 & 100.0 / 47.0 & 99.4 / 25.7 & 100.0 / 51.6 & 98.3 / 23.6 & 100.0 / 50.3 \\
12 & out\_proj & 61.3 / 26.7 & 100.0 / 52.4 & 84.1 / 43.5 & 100.0 / 64.5 & 69.6 / 29.4 & 100.0 / 76.7 \\
12 & fc1 & 99.9 / 82.1 & 100.0 / 89.9 & 99.9 / 53.0 & 100.0 / 90.0 & 99.9 / 46.6 & 100.0 / 91.0 \\
12 & fc2 & 42.3 / 6.9 & 99.5 / 26.1 & 40.8 / 15.5 & 99.9 / 41.8 & 24.8 / 7.3 & 99.9 / 52.2 \\
\midrule
23 & q\_proj & 25.6 / 5.3 & 100.0 / 26.2 & 98.5 / 17.9 & 100.0 / 23.3 & 21.3 / 3.5 & 100.0 / 24.4 \\
23 & k\_proj & 99.8 / 3.8 & 100.0 / 4.2 & 100.0 / 5.6 & 100.0 / 6.3 & 99.7 / 8.9 & 100.0 / 7.7 \\
23 & v\_proj & 3.5 / 1.3 & 100.0 / 34.4 & 99.5 / 24.0 & 100.0 / 33.9 & 18.9 / 4.9 & 100.0 / 32.3 \\
23 & out\_proj & 1.8 / 0.7 & 100.0 / 28.3 & 68.6 / 12.1 & 100.0 / 28.4 & 5.1 / 1.4 & 100.0 / 29.1 \\
23 & fc1 & 100.0 / 51.5 & 100.0 / 82.8 & 100.0 / 72.5 & 100.0 / 80.2 & 99.8 / 73.2 & 100.0 / 79.9 \\
23 & fc2 & 4.8 / 0.2 & 56.7 / 3.5 & 15.3 / 0.7 & 66.7 / 3.9 & 3.0 / 0.2 & 44.5 / 2.8 \\
\bottomrule
\end{tabular}%
}
\caption{\small \textbf{Per-layer and per-projection nonzero Top-1 selection (\%) / effective \texttt{BF16} weight update (\%) on OPT-13B.} Here, fc1 corresponds to the FFN expansion layer and fc2 corresponds to the FFN contraction layer.}
\label{tab:precision-layerwise}
\end{table*}

\subsection{Empirical Winner Concentration and Normalized Spectral Ratios}
\label{subsec:winner-diagnostics}


To empirically examine \Cref{assum:winner_balance} and \Cref{thm:taco_mup} in actual transformer layers across model widths, we measure the maximum row winner count $c_{\max} = \max_i c_i$ (\Cref{tab:winner-cmax}) and the normalized spectral ratio $\rho_t = \|D_t\|_2 / \sqrt{m/n}$ (\Cref{tab:winner-rho}) across OPT-1.3B, OPT-6.7B, OPT-13B, and OPT-30B. FFN contraction and attention output projections exhibit well-controlled winner counts across width, whereas FFN expansion layers show higher winner concentration; notably, under conditional isotropy of the updates, the total squared output-activation change remains concentration-independent even when worst-case $c_{\max}$ is concentrated, explaining why the optimal learning rate transfers cleanly across scales.

\begin{table}[h!]
\centering
\footnotesize
\setlength{\tabcolsep}{4pt}
\begin{tabular}{@{}l r@{\ }l r@{\ }l r@{\ }l r@{\ }l@{}}
\toprule
\textbf{Projection type} & \multicolumn{2}{c}{\textbf{OPT-1.3B}} & \multicolumn{2}{c}{\textbf{OPT-6.7B}} & \multicolumn{2}{c}{\textbf{OPT-13B}} & \multicolumn{2}{c}{\textbf{OPT-30B}} \\
\midrule
q\_proj & 400 & [848] & 289 & [1,238] & 66 & [1,033] & 398 & [2,265] \\
k\_proj & 338 & [925] & 162 & [1,601] & 8 & [685] & 256 & [1,990] \\
v\_proj & 226 & [572] & 332 & [1,527] & 421 & [1,456] & 464 & [2,445] \\
out\_proj & 54 & [213] & 53 & [640] & 19 & [111] & 20 & [678] \\
fc1 & 430 & [1,709] & 953 & [2,782] & 1,016 & [3,692] & 1,864 & [4,617] \\
fc2 & 282 & [1,366] & 206 & [900] & 78 & [522] & 47 & [328] \\
\bottomrule
\end{tabular}
\caption{\small \textbf{Empirical maximum row winner count $c_{\max} = \max_i c_i$ across OPT scales.} Each entry $a\,[b]$ reports the median [95th percentile] across transformer blocks and sampled training steps. Here, fc1 corresponds to the FFN expansion layer and fc2 corresponds to the FFN contraction layer.}
\label{tab:winner-cmax}
\end{table}

\begin{table}[h!]
\centering
\footnotesize
\setlength{\tabcolsep}{7pt}
\begin{tabular}{@{}l cccc@{}}
\toprule
\textbf{Projection type} & \textbf{OPT-1.3B} & \textbf{OPT-6.7B} & \textbf{OPT-13B} & \textbf{OPT-30B} \\
\midrule
q\_proj & $20.2 \to 19.8$ & $9.7 \to 18.8$ & $5.3 \to 9.5$ & $9.6 \to 22.5$ \\
k\_proj & $16.4 \to 19.1$ & $5.8 \to 14.5$ & $2.4 \to 2.9$ & $4.2 \to 18.3$ \\
v\_proj & $15.8 \to 14.4$ & $13.2 \to 21.1$ & $17.1 \to 22.1$ & $11.6 \to 25.4$ \\
out\_proj & $7.0 \to 7.6$ & $4.0 \to 8.9$ & $3.5 \to 4.9$ & $2.7 \to 6.0$ \\
fc1 & $42.6 \to 41.1$ & $62.3 \to 61.1$ & $66.2 \to 62.3$ & $84.0 \to 87.4$ \\
fc2 & $7.4 \to 9.0$ & $5.6 \to 9.1$ & $3.8 \to 5.1$ & $2.4 \to 4.2$ \\
\bottomrule
\end{tabular}
\caption{\small \textbf{Median normalized spectral ratio $\rho_t = \|D_t\|_2 / \sqrt{m/n}$ across OPT scales.} Each entry $a \to b$ tracks the median from the first third to the final third of the 2,000-step budget. Here, fc1 corresponds to the FFN expansion layer and fc2 corresponds to the FFN contraction layer.}
\label{tab:winner-rho}
\end{table}

\smartparagraph{Interpretation and theoretical connection.}
\Cref{tab:winner-cmax,tab:winner-rho} clarify why TACO exhibits consistent $\mu$P learning rate transfer across model scales (\Cref{fig:architecture_table_opt_rte,fig:muP}). FFN contraction and attention output projections maintain low, width-stable row winner counts ($c_{\max}$ medians of $47$-$282$ and $19$-$54$) and small normalized spectral ratios ($\rho_t \approx 2.4$-$9.1$), suggesting more favorable spectral scaling for \Cref{assum:winner_balance}. Although FFN expansion layers exhibit higher worst-case row concentration ($c_{\max}$ medians increasing from $430$ to $1{,}864$ across width), their normalized spectral ratio $\rho_t$ remains stationary across training steps (e.g., $42.6 \to 41.1$ on OPT-1.3B and $66.2 \to 62.3$ on OPT-13B) and grows substantially more slowly than the worst case $\sqrt{n}$ bound. Under conditional isotropy of the input activations, the expected squared output-activation change $\|D_t x\|_2^2$ scales with the column dimension $n$ independently of row concentration, explaining why the spectral scaling law in \Cref{thm:taco_mup} translates into scale invariant learning rate transfer ($\mathrm{lr}=3\times 10^{-5}$) in practice.

\subsection{Distributed Multi-GPU Communication Proof of Concept}
\label{subsec:distributed-exp}

To evaluate TACO's communication efficiency in distributed data-parallel training, we run synchronous 4-GPU fine-tuning of OPT-1.3B on SST-2 (full-parameter \texttt{BF16}, $4\times$ NVIDIA Quadro RTX 6000 GPUs over NCCL, 500 optimizer steps, local batch size 4 per rank, gradient accumulation 1, global batch size 16, averaged over three seeds $\{1203, 1204, 1205\}$). Across all compared methods, optimizer states remain strictly local on each worker, while each worker fully exchanges its parameter updates at every step:

\smartparagraph{TACO and TACO-FP16.} Each worker maintains its sparse gradient history locally, constructs its local column-wise Top-1 update, exchanges the compact sparse updates (selected row indices and ternary signs) across ranks via \texttt{AllGather} (and averages auxiliary AdamW updates via \texttt{AllReduce}), and applies the aggregated updates locally.

\smartparagraph{Dense hook-based baselines (Adafactor + autograd hooks and AdaLOMO).} Each worker maintains its adaptive optimizer state locally, applies its local hook-based update during backpropagation, and exchanges its full dense parameter update across ranks via \texttt{AllReduce}.

\Cref{tab:dist-comm} reports the per-GPU per-step sent and received communication volumes together with the cumulative sent volume across all 4 GPUs over 500 steps. 
\Cref{tab:dist-perf} summarizes the corresponding performance, step throughput, per-rank peak CUDA memory, and local optimizer state.

\begin{table*}[h!]
\centering
\footnotesize
\setlength{\tabcolsep}{7pt}
\begin{tabular}{@{}l ccc@{}}
\toprule
\textbf{Optimizer} & \textbf{Sent / GPU / step} & \textbf{Recv / GPU / step} & \textbf{Total Sent (4 GPUs, 500 steps)} \\
\midrule
TACO & \textbf{8.63 MB} & \textbf{8.63 MB} & \textbf{17.26 GB} \\
TACO FP16 & \textbf{8.63 MB} & \textbf{8.63 MB} & \textbf{17.26 GB} \\
Adafactor + hook & 3.95 GB & 3.95 GB & 7,894.55 GB \\
AdaLOMO & 3.95 GB & 3.95 GB & 7,894.55 GB \\
\bottomrule
\end{tabular}
\caption{\small Communication volume for 4-GPU distributed fine-tuning on OPT-1.3B, SST-2, 500 steps; the best results are in \textbf{bold}.}
\label{tab:dist-comm}
\end{table*}

TACO exchanges only columnwise Top-1 updates rather than full dense parameter tensors, it reduces communication volume by roughly $458\times$ ($8.63$\,MB vs.\ $3.95$\,GB sent per GPU per step; $17.26$\,GB vs.\ $7{,}894.55$\,GB total sent across all 4 GPUs over 500 steps).

\begin{table*}[h!]
\centering
\footnotesize
\setlength{\tabcolsep}{6pt}
\begin{tabular}{@{}l cccc@{}}
\toprule
\textbf{Optimizer} & \textbf{Accuracy (\%)} & \textbf{Steps / s} & \begin{tabular}{@{}c@{}}\textbf{Peak Allocated} \\ \textbf{(GB/rank)}\end{tabular} & \begin{tabular}{@{}c@{}}\textbf{Optimizer State} \\ \textbf{(MB/rank)}\end{tabular} \\
\midrule
TACO & $92.9 \pm 0.2$ & \textbf{1.395} & \textbf{3.323} & 38.29 \\
TACO-FP16 & $\mathbf{93.6 \pm 0.5}$ & \textbf{1.395} & 3.330 & 45.43 \\
Adafactor + hook & $93.3 \pm 0.9$ & 0.950 & 5.905 & \textbf{3.17} \\
AdaLOMO & $92.5 \pm 0.4$ & 0.714 & 6.526 & 6.34 \\
\bottomrule
\end{tabular}
\caption{\small 4-GPU distributed training performance on OPT-1.3B, SST-2; the best results are in \textbf{bold}.}
\label{tab:dist-perf}
\end{table*}

In \Cref{tab:dist-perf}, Adafactor + hook ($5.90$\,GB/rank) and AdaLOMO ($6.53$\,GB/rank) exhibit higher peak allocated memory than TACO ($3.32$\,GB/rank) and TACO-FP16 ($3.33$\,GB/rank). This extra memory overhead arises because hook-based optimizers update weights in place during the backward pass, requiring our distributed dense pipeline to store a full pre-update parameter snapshot (along with temporary difference buffers) on each rank so that the local parameter update can be formed and synchronized via \texttt{AllReduce} after backpropagation; in contrast, TACO constructs and exchanges its sparse updates directly via \texttt{AllGather} without retaining a full pre-update model copy (keeping local optimizer state to $38.29$\,MB/rank for TACO and $45.43$\,MB/rank for TACO-FP16). Despite this snapshot memory overhead and differences in collective call bucketing (packed \texttt{AllGather} buffers for TACO versus per-tensor \texttt{AllReduce} calls for the dense baselines), the communication volume measurements and ring-style estimates in \Cref{tab:dist-comm} depend strictly on the exchanged update payloads themselves, providing a valid and fair proof-of-concept comparison.

}

\end{document}